\documentclass{article}

\PassOptionsToPackage{numbers, compress}{natbib}
\usepackage[preprint]{neurips_2026}

\usepackage[utf8]{inputenc} 
\usepackage[T1]{fontenc}    
\usepackage{hyperref}       
\usepackage{url}            
\usepackage{booktabs}       
\usepackage{amsfonts}       
\usepackage{nicefrac}       
\usepackage{microtype}      
\usepackage{xcolor}         
\usepackage[utf8]{inputenc} 
\usepackage[T1]{fontenc}    
\usepackage{hyperref}       
\usepackage{url}            
\usepackage{booktabs}       
\usepackage{amsfonts}       
\usepackage{nicefrac}       
\usepackage{microtype}      
\usepackage{xcolor}         
\usepackage{todonotes}
\usepackage{xargs}  
\usepackage{graphicx}
\usepackage{pgfplots}
\usepackage{pgfplotstable}
\tikzset{>=stealth}
\usetikzlibrary{patterns}
\usetikzlibrary{pgfplots.statistics}
\usetikzlibrary{backgrounds,scopes}
\usetikzlibrary{backgrounds,scopes}
\usetikzlibrary{arrows}
\usetikzlibrary{positioning}
\usetikzlibrary{calc}
\usetikzlibrary{fit}
\usetikzlibrary{shapes.misc}
\usetikzlibrary{external}
\usepackage{amsthm}

\usepackage{booktabs}
\usepackage{xcolor}
\usepackage{colortbl}
\usepackage{array}
\usepackage{multirow}
\usepackage{makecell}

\newtheorem{theorem}{Theorem}
\newtheorem{lemma}{Lemma}
\newtheorem{example}{Example}
\newtheorem{corollary}{Corollary}
\newtheorem{definition}{Definition}
\newtheorem{assumption}{Assumption}
\newtheorem{remark}{Remark}
\newtheorem{proposition}{Proposition}

\usepackage{moreverb}

\usepackage{listings}

\usepackage{amsmath}
\usepackage{amssymb}
\usepackage{amsfonts}
\usepackage{upgreek}
\usepackage{dsfont}
\usepackage{accents}
\usepackage{mathtools, nccmath}
\usepackage{svg}
\usepackage{tabularx}
\usepackage{enumerate}
\usepackage{booktabs}
\usepackage{pifont}
\newcommand{\cmark}{\ding{51}}

\newcommand{\na}{$-^{\dagger}$}
\usepackage{tikz}
\usetikzlibrary{arrows.meta, positioning, fit}
\usepackage{xcolor, colortbl, makecell, array, amsmath}
\usepackage{amsmath}
\usepackage{subcaption}
\usepackage{amssymb}
\usepackage{mathtools}
\usepackage{amsthm}
\usepackage{ulem}
\usepackage{tikz}
\usepackage{xcolor, colortbl, makecell, array, amsmath, booktabs}
\usetikzlibrary{matrix, positioning, calc, shapes.geometric, arrows.meta, fit, backgrounds, patterns, decorations.pathreplacing}
\usepackage{multirow}
\usepackage{array}
\usepackage{graphicx}

\DeclareMathOperator*{\argmin}{arg\,min}

\definecolor{headerblue}{RGB}{220,230,242}
\definecolor{cellbg}{RGB}{248,248,248}
\definecolor{cellblue}  {RGB}{220, 235, 252}   
\definecolor{cellgreen} {RGB}{210, 241, 226}   
\definecolor{hdr}       {RGB}{40,  80,  140}   
\definecolor{hdrtxt}    {RGB}{255, 255, 255}   
\definecolor{spanbg}    {RGB}{245, 245, 245}   

\definecolor{contrib}{RGB}{220, 235, 252}   
\definecolor{future}{RGB}{245, 245, 245}    
\definecolor{headerrow}{RGB}{40, 80, 140}   
\definecolor{headertext}{RGB}{255,255,255}  

\definecolor{contrib}   {RGB}{220, 235, 252}   
\definecolor{newcontrib}{RGB}{210, 241, 226}   
\definecolor{hdr}       {RGB}{40,  80,  140}   
\definecolor{hdrtxt}    {RGB}{255, 255, 255}   
 
\tikzset{
  ev/.style  = {circle, draw=blue!60,   fill=blue!13,
                minimum size=12pt, inner sep=0pt, font=\tiny},
  yn/.style  = {circle, draw=orange!80, fill=orange!22,
                minimum size=12pt, inner sep=0pt, font=\tiny},
  gn/.style  = {circle, draw=teal!70,   fill=teal!17,
                minimum size=12pt, inner sep=0pt, font=\tiny},
  sev/.style = {circle, draw=blue!60,   fill=blue!13,
                minimum size=9pt,  inner sep=0pt, font=\tiny},
  syn/.style = {circle, draw=orange!80, fill=orange!22,
                minimum size=9pt,  inner sep=0pt, font=\tiny},
  ca/.style  = {->, >=Stealth, thick, blue!60},
  nc/.style  = {->, >=Stealth, dashed, gray!55},
  ag/.style  = {double equal sign distance, ->, >=Stealth, thick, gray!60},
}

\tikzset{
  ev/.style  = {circle, draw=blue!60,  fill=blue!12,
                minimum size=13pt, inner sep=0pt, font=\scriptsize},
  yn/.style  = {circle, draw=orange!80,fill=orange!20,
                minimum size=13pt, inner sep=0pt, font=\scriptsize},
  gn/.style  = {circle, draw=teal!70,  fill=teal!15,
                minimum size=13pt, inner sep=0pt, font=\scriptsize},
  sev/.style = {circle, draw=blue!60,  fill=blue!12,
                minimum size=9pt,  inner sep=0pt, font=\tiny},
  syn/.style = {circle, draw=orange!80,fill=orange!20,
                minimum size=9pt,  inner sep=0pt, font=\tiny},
  ca/.style  = {->, >=Stealth, thick,  blue!65},
  nc/.style  = {->, >=Stealth, dashed, gray!55},
  ag/.style  = {double equal sign distance, ->, >=Stealth,
                thick, gray!55, shorten <=2pt, shorten >=2pt},
}

\definecolor{codebg}{RGB}{245,245,244}
\definecolor{codekw}{RGB}{0,0,180}
\definecolor{codestring}{RGB}{153,0,0}
\definecolor{codecomment}{RGB}{0,128,0}

\lstdefinestyle{oscar}{
  backgroundcolor=\color{codebg},
  basicstyle=\ttfamily\footnotesize,
  keywordstyle=\color{codekw}\bfseries,
  stringstyle=\color{codestring},
  commentstyle=\color{codecomment}\itshape,
  columns=fullflexible,
  frame=single,
  framerule=0pt,
  rulecolor=\color{black!20},
  breaklines=true,
  showstringspaces=false,
  tabsize=2,
}
\lstdefinestyle{mystyle}{
    commentstyle=\color{codegreen},
    keywordstyle=\color{magenta},
    numberstyle=\tiny\color{codegray},
    stringstyle=\color{codepurple},
    basicstyle=\ttfamily\footnotesize,
    breakatwhitespace=false,         
    breaklines=true,                 
    captionpos=b,                    
    keepspaces=true,                 
    numbers=left,                    
    numbersep=5pt,                  
    showspaces=false,                
    showstringspaces=false,
    showtabs=false,                  
    tabsize=2
}

\tikzset{
    -Latex,auto,node distance =1 cm and 1 cm,semithick,
    state/.style ={ellipse, draw, minimum width = 0.7 cm},
    point/.style = {circle, draw, inner sep=0.04cm,fill,node contents={}},
    bidirected/.style={Latex-Latex,dashed},
    el/.style = {inner sep=2pt, align=left, sloped}
}

\title{\textsc{seq2cause}: One Autoregressive Backbone, Four Causal Discovery Tasks in Event Sequences}

\author{%
  Hugo Math \\
  \texttt{hugo.math@bmwgroup.com} \\
}

\begin{document}

\maketitle

\begin{abstract}
    Complex systems — vehicles, patients, genomes — emit discrete event sequences whose operative question is causal, not predictive: which events cause which other events, and which cause higher-level outcomes such as failures or diseases? This question decomposes along two axes — dependency type (event $\to$ event vs. event $\to$ outcome) and causal scope (single sequence vs. population) — yielding four structurally distinct regimes with different identifiability conditions. No existing method addresses more than one, because all assume multi-stream structure with low vocabulary, and none scales beyond a few hundred event types.
We present $\textsc{seq2cause}$, a unified framework that resolves all four regimes through a single shared primitive: a pretrained autoregressive model repurposed as an amortized conditional independence testing engine requiring no task-specific retraining. We establish a prediction–causality duality: the model's excess cross-entropy simultaneously bounds causal identification error across all four regimes, so that every improvement in next-token prediction tightens causal guarantees for free.
On nonlinear SCMs (vocabularies up to $8,000$ types) and real-world vehicle diagnostic logs ($29$K event types, $474$ failure outcomes), $\textsc{seq2cause}$ is the first method to populate all four regimes at scale with a single frozen backbone. Existing methods are either inapplicable, inaccurate, or computationally intractable in this setting.
\end{abstract}

\section{Introduction} 
\begin{figure}[!b]
    \centering
\resizebox{0.97\textwidth}{!}{%
\begin{tikzpicture}[
    font=\sffamily\small,
    redbox/.style={draw, rounded corners=3pt, fill=red!6,
                   draw=red!55!black, minimum height=0.72cm,
                   align=center, inner sep=5pt},
    graybox/.style={draw, rounded corners=3pt, fill=gray!10,
                    minimum height=0.72cm, align=center, inner sep=5pt},
    bluecond/.style={draw, rounded corners=3pt, fill=blue!7,
                     draw=blue!50!black, minimum height=0.60cm,
                     align=center, inner sep=4pt},
    orangecond/.style={draw, rounded corners=3pt, fill=orange!9,
                       draw=orange!55!black, minimum height=0.60cm,
                       align=center, inner sep=4pt},
    regime/.style={draw, rounded corners=2pt,
                   minimum width=1.80cm, minimum height=0.62cm,
                   align=center, font=\sffamily\scriptsize, inner sep=2pt},
    arr/.style={->, thick, red!50!gray},
    bluearr/.style={->, thick, blue!55!black},
    orangearr/.style={->, thick, orange!60!black},
    greenarr/.style={->, thick, green!55!black}
]

\node[graybox, minimum width=2.5cm, minimum height=3.0cm]
    (input) at (0,0) {};
\node[draw, rounded corners=2pt, fill=blue!7, draw=blue!40!black,
      minimum width=2.1cm, minimum height=1.0cm, align=center,
      font=\sffamily\scriptsize]
    (evbox) at (0, 0.55)
    {\textbf{Events}\\[2pt]$x_1\;\; x_2\;\; x_3\;\; x_4$};
\node[draw, rounded corners=2pt, fill=orange!9, draw=orange!45!black,
      minimum width=2.1cm, minimum height=0.8cm, align=center,
      font=\sffamily\scriptsize]
    (outbox) at (0, -0.65)
    {\textbf{Outcomes}\\[1pt]
     $\boldsymbol{y}$~{\color{gray}\scriptsize(at $t{=}L$)}};
\node[font=\sffamily\scriptsize\bfseries] at (0, 1.30)
    {\textbf{Input sequence(s)}};
\node[font=\sffamily\scriptsize, text=gray] at (0, -1.28)
    {pop.:~$\times n$ i.i.d.};

\node[redbox, minimum width=2.8cm, minimum height=3.0cm]
    (lm) at (3.45, 0)
    {\textbf{AR Backbone}\\$P_\theta$\\[3pt]
     \scriptsize LLaMA, GPT,\\
     \scriptsize RNN, Mamba\\[3pt]
     \scriptsize no fine-tuning};

\foreach \y in {0.9, 0.3, -0.3, -0.9}{
    \draw[arr] ([yshift=\y cm]input.east) --
               ([yshift=\y cm]lm.west);
}

\node[bluecond, minimum width=3.2cm]
    (px) at (7.6, 0.80)
    {\color{blue!70!black}$P_\theta(X_{t+1} \mid x_{\leq t})$\\
     \scriptsize\color{blue!60!black}next-event distribution};

\node[orangecond, minimum width=3.2cm]
    (py) at (7.6, -0.80)
    {\color{orange!75!black}$P_\theta(Y \mid x_{\leq t})$\\
     \scriptsize\color{orange!65!black}outcome likelihood};

\foreach \dy in {0.27, 0.09, -0.09, -0.27}{
    \draw[bluearr, opacity=0.7]
        ([yshift=\dy cm]lm.east|-px.center) --
        ([yshift=\dy cm]px.west);
    \draw[orangearr, opacity=0.7]
        ([yshift=\dy cm]lm.east|-py.center) --
        ([yshift=\dy cm]py.west);
}

\node[redbox, minimum width=2.6cm]
    (ci) at (11.3, 0.80)
    {\textbf{Parallel CMI}\\
     \scriptsize$\hat{I}_N$};

\draw[bluearr]   (px.east) -- (ci.west);
\draw[orangearr] (py.east) to[out=0, in=225] (ci.south west);

\node[redbox, minimum width=2.6cm]
    (mdl) at (11.3, -0.80)
    {\textbf{MDL Criterion}\\
     \scriptsize threshold $\tau^*_{uv}$};

\draw[orangearr] (py.east)  -- (mdl.west);
\draw[bluearr]   (px.east)  to[out=0, in=135] (mdl.north west);

\node[regime, fill=blue!10, draw=blue!50!black]
    (ca) at (15.0, 0.80)
    {\textbf{(a)}~$\mathcal{G}^s_{\mathcal{XX}}$\\
     \scriptsize sample $\mathcal{X}\to\mathcal{X}$};

\node[regime, fill=orange!10, draw=orange!50!black,
      below=0.20cm of ca] (cb)
    {\textbf{(b)}~$\mathcal{G}^s_{\mathcal{XY}}$\\
     \scriptsize sample $\mathcal{X}\to\mathcal{Y}$};

\draw[bluearr]   (ci.east) -- (ca.west);
\draw[orangearr] (ci.east) to[out=0, in=180] (cb.west);

\node[regime, fill=green!9, draw=green!55!black,
      right=0.45cm of ca] (rcas)
    {\textbf{Sample RCA}};

\draw[bluearr]   (ca.east) -- (rcas.west);
\draw[orangearr] (cb.east) to[out=0, in=225] (rcas.south west);

\node[regime, fill=blue!6, draw=blue!40!black]
    (cc) at (15.0, -0.80)
    {\textbf{(c)}~$\mathcal{G}_{\mathcal{XX}}$\\
     \scriptsize pop.\ $\mathcal{X}\to\mathcal{X}$};

\node[regime, fill=orange!6, draw=orange!40!black,
      below=0.20cm of cc] (cd)
    {\textbf{(d)}~$\mathcal{G}_{\mathcal{XY}}$\\
     \scriptsize pop.\ $\mathcal{X}\to\mathcal{Y}$};

\draw[bluearr]   (mdl.east) -- (cc.west);
\draw[orangearr] (mdl.east) to[out=0, in=180] (cd.west);

\node[regime, fill=green!7, draw=green!45!black,
      right=0.45cm of cc] (rcap)
    {\textbf{Pop.\ RCA}};

\draw[bluearr]   (cc.east) -- (rcap.west);
\draw[orangearr] (cd.east) to[out=0, in=225] (rcap.south west);

\end{tikzpicture}
}
\caption{\textbf{\textsc{seq2cause}: one frozen backbone, 
four causal regimes.}
Any pretrained AR model — LLaMA, GPT, Mamba — is repurposed 
via teacher forcing as a parallelised causal discovery engine on GPUs, 
with no fine-tuning.
Two heads share the same backbone: the event conditional 
(\textcolor{blue!65!black}{blue}) and the outcome likelihood 
(\textcolor{orange!75!black}{orange}). Using a single sequence, the
sample-level graphs~(a,~b) follow from CMI testing; for multiple sequences, 
population-level graphs~(c,~d) are obtained from Minimum Description Length (MDL)~\cite{grunwald2004tutorialintroductionminimumdescription}. Their union enables root-cause identification in sequences.}
\label{fig:seq2cause}
\end{figure}
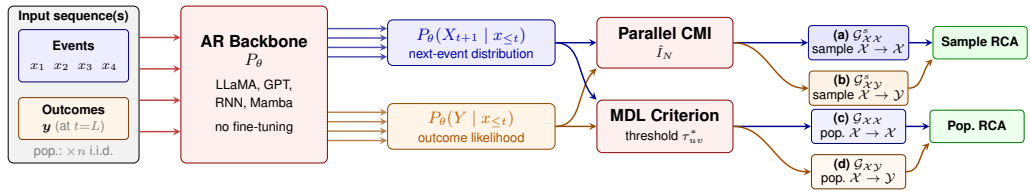

Discrete event sequences are the canonical data modality of modern complex systems. A vehicle emits hundreds of diagnostic trouble codes before a failure \cite{pdm_dtc_feature_extraction, math2024harnessingeventsensorydata,math2026contextinformed}. A patient's trajectory through a hospital unfolds as a sequence of interventions, diagnoses, and observations \cite{Li2021CausalHM, bihealth}. A genome expresses itself as an ordered chain of regulatory events~\cite{BEENE, SHEN2023106536, Theodoris2023TransferLE, Cui2024}. In each of these domains, the operative scientific question is not predictive but causal: which events cause which other events, and which events cause higher-level outcomes such as failures, diseases, or phenotypes?

This question decomposes into four structurally distinct regimes (Tab. \ref{tab:framework}), defined by two orthogonal axes. The first is \textit{dependency type}: one may seek event-to-event structure (\(\mathcal{X} \to \mathcal{X}\)) — how events cause other events — or event-to-outcome structure (\(\mathcal{X} \to \mathcal{Y}\)) — how events cause final outcomes. The second is \textit{causal scope}: inference may target a single observed sequence (sample-level), where the causal object is specific to one sample, or an entire dataset (population-level), where the goal is a type-level summary graph generalizable across sequences. These four regimes are not interchangeable: the sample-level question "did `Fire@t1` cause `Smoke@t2`?" and the population-level question "does `Fire` cause `Smoke` in general?" have different identifiability conditions, different estimands, and different failure modes. Surprisingly, no existing framework addresses more than one \cite{assaad_survey_ijcai_cd_time_series, cd_temporaldata_review}.

The dominant paradigm for event sequence causal discovery — Granger-based \cite{granger_causality}, Hawkes-based \cite{granger_causality_hawkes, shtp, hawke_process_and_app_2025} and constraint-based methods \cite{constrainct_based_cd}, operates on \(|\mathcal{X}|\) parallel streams, each assigned to a distinct entity. This design is structurally mismatched with the single-stream, high-vocabulary setting. When \(|\mathcal{X}|\to \infty\), the expected co-occurrence count of any pair \((u,v)\) in a sequence of length \(L\) vanishes, making empirical joint distributions degenerate. The Hawkes excitation matrix alone carries \(\mathcal{O}(|\mathcal{X}|^2)\) parameters, rendering maximum likelihood estimation intractable above \(|\mathcal{X}|\sim 10^3\). On the \(\mathcal{X} \to \mathcal{Y}\) side, classical Markov boundary (MB) algorithms (IAMB \cite{iamb}, CMB \cite{cmb}, MI-MCF \cite{mimcf}) face the same dimensionality barrier and have no clear application to sequential data. The single-stream, high-vocabulary regime is thus left entirely unaddressed.
\begin{table}[!t]
\centering
\caption{\textbf{The four regimes of causal discovery in discrete 
event sequences.}
Rows separate causal scope — a single observed sequence 
(sample-level) vs.\ an entire dataset (population-level).
Columns separate dependency type — event-to-event 
($\mathcal{X}\!\to\!\mathcal{X}$, \textcolor{blue!65!black}{blue}) 
vs.\ event-to-outcome 
($\mathcal{X}\!\to\!\mathcal{Y}$, \textcolor{orange!75!black}{orange}).
These four cells have different identifiability conditions, 
different estimands, and different failure modes.
No existing method addresses more than one at scale.}
\label{tab:framework}

\renewcommand{\arraystretch}{1.0}
\setlength{\tabcolsep}{4pt}
\begin{tabular}{
    >{\centering\arraybackslash}m{2.5cm}
  | >{\centering\arraybackslash}m{5.5cm}
  | >{\centering\arraybackslash}m{5.5cm}
}

\rowcolor{hdr}
\textcolor{hdrtxt}{}
& \textcolor{hdrtxt}{\makecell[c]{
    \textbf{Event-to-Event: }\small$\mathcal{X}\!\to\!\mathcal{X}$}}
& \textcolor{hdrtxt}{\makecell[c]{
    \textbf{Events-to-Outcomes: }\small$\mathcal{X}\!\to\!\mathcal{Y}$}}
\\\hline

\cellcolor{hdr}\textcolor{hdrtxt}{\makecell[c]{
  \textbf{Sample-Level}\\[3pt]\small(Single Sequence)}}
& \cellcolor{cellblue}
  \begin{minipage}[c]{5.2cm}\centering
  \begin{tikzpicture}[baseline=-0.5ex, scale=0.88]
    \node[ev] (t1) at (0,    0) {$A$};
    \node[ev] (t2) at (0.75, 0) {$B$};
    \node[ev] (t3) at (1.50, 0) {$A$};
    \node[ev] (t4) at (2.25, 0) {$C$};
    \draw[ca] (t1) -- (t2);
    \draw[ca] (t1) to[bend left=40] (t3);
    \draw[ca] (t3) -- (t4);
    \draw[nc] (t2) -- (t3);
    \foreach \n/\lab in {t1/$t_1$, t2/$t_2$, t3/$t_3$, t4/$t_4$}{
      \node[font=\tiny, gray!55, below=3pt of \n] {\lab};
    }
    \draw[->, gray!50, thin] (-0.28,-0.75) -- (2.53,-0.75)
      node[right, font=\tiny, gray!55]{$t$};
  \end{tikzpicture}
  \par
  {\small\textbf{(a) Sample Time Graph Recovery}}\par
  \vspace{3pt}
  {\scriptsize Recover the time causal graph specific to a single
  sequence of events with lagged effects.}
  \end{minipage}
& \cellcolor{cellblue}
  \begin{minipage}[c]{5.2cm}\centering
  \begin{tikzpicture}[baseline=-0.5ex, scale=0.88]
    \node[ev] (x1) at (0,    0) {$A$};
    \node[ev] (x2) at (0.75, 0) {$B$};
    \node[ev] (x3) at (1.50, 0) {$C$};
    \node[yn] (y)  at (2.25, 0) {$y$};
    \draw[ca] (x1) to[bend left=40] (y);
    \draw[ca] (x3) -- (y);
    \node[draw=orange!65, dashed, thick, rounded corners=2pt,
          inner sep=3pt, fit=(x1)] {};
    \node[draw=orange!65, dashed, thick, rounded corners=2pt,
          inner sep=3pt, fit=(x3)] {};
    \node[font=\tiny, orange!75!black, above=6pt of x2] {$\mathcal{MB}(y)$};
  \end{tikzpicture}
  \par
  {\small\textbf{(b) Sample Markov Boundary}}\par
  \vspace{3pt}
  {\scriptsize Identify the direct causes of each outcome $y$
  (e.g., defects, diseases) from a single sequence.}
  \end{minipage}
\\\hline   

\cellcolor{hdr}\textcolor{hdrtxt}{\makecell[c]{
  \textbf{Population-Level}\\[3pt]\small(Full Dataset)}}
& \cellcolor{cellblue}
  \begin{minipage}[c]{5.2cm}\centering
  \begin{tikzpicture}[baseline=-0.5ex, scale=0.88]
    \node[sev] (i1a) at (0,    0.55) {$A$};
    \node[sev] (i1b) at (0.60, 0.55) {$B$};
    \draw[ca, thin] (i1a) -- (i1b);
    \node[font=\tiny, gray!55] at (-0.5, 0.55) {$s^{(1)}$};
    \node[sev] (i2a) at (0,    -0.10) {$A$};
    \node[sev] (i2c) at (0.60, -0.10) {$C$};
    \draw[ca, thin] (i2a) -- (i2c);
    \node[font=\tiny, gray!55] at (-0.5,-0.10) {$s^{(2)}$};
    \draw[ag] (0.82, 0.22) -- (1.5, 0.22);
    \node[sev] (sa) at (1.60,  0.55)  {$A$};
    \node[sev] (sb) at (2.35,  0.55)  {$B$};
    \node[sev] (sc) at (1.98, -0.10)  {$C$};
    \draw[ca] (sa) -- (sb);
    \draw[ca] (sa) -- (sc);
  \end{tikzpicture}
  \par\vspace{2pt}
  {\small\textbf{(c) Population Summary Graph}}\par
  \vspace{3pt}
  {\scriptsize Summary causal graph over event types $\mathcal{X}$:
  \textit{which event caused other events across samples?}}
  \end{minipage}
& \cellcolor{cellblue}
  \begin{minipage}[c]{5.2cm}\centering
  \begin{tikzpicture}[baseline=-0.5ex, scale=0.88]
  \node[sev] (a1) at (0,    0.55)  {$A$};
  \node[sev] (b1) at (0.55, 0.55)  {$B$};
  \node[syn] (y1) at (1.10, 0.55)  {$y$};
  \draw[ca, thin] (a1) -- (b1);
  \draw[ca, thin] (b1) -- (y1);
  \node[font=\tiny, gray!55] at (-0.5, 0.55) {$s^{(1)}$};
  \node[sev] (a2) at (0,    -0.10) {$A$};
  \node[sev] (c2) at (0.55, -0.10) {$C$};
  \node[syn] (y2) at (1.10, -0.10) {$y$};
  \draw[ca, thin] (a2) -- (c2);
  \draw[ca, thin] (c2) -- (y2);
  \node[font=\tiny, gray!55] at (-0.5,-0.10) {$s^{(2)}$};
  \draw[ag] (1.28, 0.22) -- (1.90, 0.22);
  \node[ev] (gb) at (2.50,  0.55)  {$B$};
  \node[ev] (gc) at (2.50, -0.10)  {$C$};
  \node[yn] (gy) at (3.4,  0.22)  {$y$};
  \draw[ca] (gb) -- (gy);
  \draw[ca] (gc) -- (gy);
  \end{tikzpicture}
  \par\vspace{2pt}
  {\small\textbf{(d) Population-Markov Boundary}}\par
  \vspace{3pt}
  {\scriptsize Identify the Causes of outcomes aggregated across sequences:
  \textit{What events lead to this product defect?}}
  \end{minipage}
\\  
\end{tabular}
\vspace{-6pt}
\end{table}

The key observation underlying this work is that an autoregressive (AR) model~\cite{seq2seq_learning, touvron2023llamaopenefficientfoundation} trained by next-token prediction implicitly estimates the full joint distribution via the chain rule \(\prod_t P(x_t | x_{<t})\). These conditionals are \textit{precisely} the quantities required to evaluate conditional mutual information (CMI), and perform conditional independence (CI) testing, without any task-specific retraining. Given a pretrained AR model with excess cross-entropy \(\epsilon\), causal discovery reduces to a batched forward pass — spanning all four regimes — with identification error vanishing at rate \(\epsilon \to 0\). \textbf{Every frontier-scale next-token checkpoint is, for free, a causal discovery engine}.
We introduce \textsc{seq2cause}\footnote{Code, trained backbones and the unified evaluation are released as the \href{https://github.com/Mathugo/seq2cause}{seq2cause} Python package.}, the first method to populate all four cells simultaneously at scale. \textsc{seq2cause} repurposes any pretrained AR model, without fine-tuning for parallelized causal discovery on GPUs. It recovers the MB of each outcome \((\mathcal{X} \to \mathcal{Y})\) and the causal graph over event types (\(\mathcal{X} \to X\)), from a single sequence at inference time (\textit{sample-level}) or aggregated across a dataset (\textit{population-level}), with guarantees derived directly from the model's pretraining quality $\epsilon$. 
Our contributions are:
\begin{itemize}
    \item \textbf{A four-regime taxonomy and its unification}. We formalize causal discovery in discrete event sequences with large vocabularies as a \(2 \times 2 \) problem defined by dependency type (\(\mathcal{X} \to \mathcal{X}\) vs. \(\mathcal{X} \to \mathcal{Y})\) and causal scope (sample-level vs. population-level). We establish that a single pretrained backbone serves all four cells without task-specific adaptation.
    \item \textbf{A prediction–causality duality}. We prove that a single quantity — the AR model's excess cross-entropy \(\epsilon\) — simultaneously controls (i) CMI estimation error (Theorem~\ref{thm:cmi_error_bound}), (ii) sample-time graph soundness (Theorem~\ref{thm:soudness_sample_level}), and (iii) MDL population edge-inclusion error (Lem.~\ref{lem:mdl_deviation}). Causal identification across all four regimes is therefore a function of pretraining quality, with no regime-specific calibration.
    \item \textbf{Empirical validation across all four regimes at scale.} On nonlinear SCMs $(|\mathcal{X}| \in [100, 8000])$ and real-world vehicle diagnostic sequences (\(|\mathcal{X}|=29,100, |\mathcal{Y}|=474\)), 
     \textsc{seq2cause} is the first method to populate all four cells of Tab~\ref{tab:framework} with a single frozen backbone where all existing methods are either inaccurate, or computationally intractable.
\end{itemize}

\section{Related Work}

Causal discovery in sequence of discrete events splits into two structurally 
distinct regimes that are not interchangeable. For a full overview, see Appendix~\ref{appendix:related_work}.

\paragraph{Multi-stream (Standard).}
The dominant paradigm operates on $|\mathcal{X}|$ parallel 
streams, 
each assigned to 
a distinct entity (sensor, user). The causal target 
is a population-level summary graph $\mathcal{G} = (\mathcal{X}, \mathcal{E})$ 
estimated from $T \to \infty$ observations \emph{per stream}. 
Methods in this family, such as Granger-based~\cite{granger_causality, 
tcdf, cause}, constraint-based~\cite{pcmci}, 
functional~\cite{varlingam}, optimization-based~\cite{dynotears}, 
and neural point processes~\cite{transformerhawkeprocess, shtp, cueppers2024causal} scale quadratically or worse with the number of events \(|\mathcal{X}|\).
For Hawkes-based methods specifically, the excitation matrix 
$\Phi \in \mathbb{R}^{|\mathcal{X}| \times |\mathcal{X}|}$ 
carries $|\mathcal{X}|^2$ parameters, and its maximum likelihood estimation  scales quadratically with the event types or worse~\cite{hawke_process_and_app_2025}, 
rendering the entire family \emph{computationally intractable} 
at $|\mathcal{X}| > 10^3$. We explain this theoretically in Appendix~\ref{derivation_ll_hawke_mutual}.

\paragraph{Single-stream (Sample- and Population-level).}
The general setting is fundamentally different. By analogy with NLP, a single sentence over a large vocabulary is the sample; a corpus is the population. We observe $n$ 
i.i.d sequences $\{s^{(k)}\}_{k=1}^n$ over a shared massive vocabulary $|\mathcal{X}| > 1000$, this is the \emph{population-level} regime. 
If during inference only one realization $s = (x_0, \cdots, x_L)$ is 
available, this is the \emph{sample-level} regime. The causal target decomposes into two subproblems: 
recovering event-to-event structure and 
event-to-outcomes structure (Tab.~\ref{tab:framework}). 
This regime is strictly harder than multi-stream for two 
compounding reasons. First, the co-occurrence frequency of 
any pair $(u, v) \in \mathcal{X}^2$ within a single sequence 
of moderate length $L$ satisfies:
\begin{equation}\label{eq:co_occ}
    \mathbb{E}\left[\#(u,v \text{ co-occur in } s)\right] 
    = \mathcal{O}\!\left(\frac{L^2}{|\mathcal{X}|^2}\right)
    \xrightarrow{|\mathcal{X}| \to \infty} 0
\end{equation}
making empirical joint distributions degenerate. Second, there is no long stationary 
trajectory to extract a population from: one vehicle trip, 
one patient visit, one genome sequence \emph{is the observation of interest}. 
Sliding-window arguments that justify multi-stream methods 
\emph{do not apply here}.  

\paragraph{Multi-label Causal Discovery.}
Multi-label causal discovery seeks to identify the Markov 
Boundary $\mathcal{MB}(Y_j)$ of each outcome label — its minimal 
conditioning set rendering $Y_j$ independent of all other 
variables~\cite{optimal_feature_set_cd}. Classical local structure learning (LSL) algorithms estimate these 
boundaries directly: IAMB~\cite{iamb}, CMB~\cite{cmb}, 
MB-by-MB~\cite{WANG2014252}, PCDbyPCD~\cite{pcdpcd} 
and MI-MCF~\cite{mimcf} all recover the \(\mathcal{MB}\) via iterative CI testing. Their extension to event sequences is obstructed by 
dimensionality ($|\mathcal{X}|^2$ candidate pairs) and distributional 
assumptions that rarely hold. Critically, none of these methods addresses the $\mathcal{X} \to \mathcal{Y}$ dependencies in the single-stream, high-vocabulary setting, nor do they handle $\mathcal{X} \to \mathcal{X}$ discovery jointly (see Appendix,  Tab~\ref{tab:related_work}). 
Similarly, existing RCA methods for event sequences (identifying primary causes of an outcome) either require a 
known system topology~\cite{chain_of_event_2024_sre}, operate 
on continuous multivariate streams~\cite{han2025root}, or 
reduce root causes to anomaly scores without a causal 
graph~\cite{assaad_rca_time_series_pmlr_2023} — none work at the sample-level contrary to \textsc{seq2cause} with Cor.~\ref{cor:rca}. 

\section{Notations and Background}\label{appendix:notations}
\paragraph{Notations}
We use capital letters (e.g., \(X\)) to denote random variables, \(P(X)\) the probability distribution of \(X\), \(P(X = x) = p(x)\) the probability of the realisation \(x\) for the random variable \(X\) and bold capital letters (e.g., \(\boldsymbol{X}\)) for sets of random variables.
We only work with discrete distributions. Let \(\mathcal{X}\) denote an alphabet and \(D_{KL}\) the \textit{Kullback-Leibler divergence} \cite{cover1999elements}. 

\paragraph{Probability space.}
Let $\Omega = \mathcal{X}^\mathbb{N}$ be the space 
of infinite sequences $\omega = (x_0, x_1, \ldots)$ 
over the finite alphabet $\mathcal{X}$.
Let $\mathcal{F} = \mathcal{B}(\mathcal{X}^\mathbb{N})$ 
be the product $\sigma$-algebra generated by 
cylinder sets, and let $P$ be the law of the 
event chain Eq.~\ref{eq:joint_dgp} on this space.

Labels $Y_1,\ldots,Y_c$ are \emph{not} part of $\Omega$: 
by the joint DGP \eqref{eq:joint_dgp}, each $Y_j$ is 
a function $Y_j : \Omega \to \{0,1\}$ defined by 
$Y_j(\omega) = \mathbf{1}[{\omega \in A_j}]$ for some 
measurable set $A_j \in \mathcal{F}$ depending on the 
full trajectory. The label head $P_{\theta_y}$ 
estimates $P(Y_j = 1 \mid X_{<t})$ 
at each step $t$ from the event prefixes.

The target variable $T_t : \Omega \to \{0,1\}$ is 
therefore, always a measurable function on 
$(\Omega, \mathcal{F}, P)$, whether 
$T_t = E_t \triangleq \mathbf{1}_{X_t = x_t}$ 
(event indicator) or $T_t = Y_j$ (outcome label).
This unified representation is what allows a single 
probability space to cover all four regimes.

\begin{table}[!h]
\centering
\caption{\textbf{Graph objects in \textsc{seq2cause}.}
  All time graphs are DAGs by temporal precedence.
  Summary graphs may contain cycles~\cite{assaad_survey_ijcai_cd_time_series}.
  The sample-time graph decomposes into two subgraphs
  ($\mathcal{G}^s_{\mathcal{XX}}$ and $\mathcal{G}^s_{\mathcal{XY}}$) by dependency type;
  their union is the complete sample-level causal graph $\mathcal{G}^s$.}
\label{tab:graph_notation}
\renewcommand{\arraystretch}{1.55}
\setlength{\tabcolsep}{5pt}
\begin{tabular}{@{}
  >{\small}p{3cm}
  >{\small}p{3.cm}
  >{\small}p{3.cm}
  >{\small}p{3.3cm}@{}
}
\toprule
\textbf{Symbol}
  & \textbf{Name}
  & \textbf{Nodes}
  & \textbf{Edge meaning} \\
\midrule
$\mathcal{G}^s_{\mathcal{XX}}$
  & Sample-time\newline event-to-event graph
  & Time indices\newline $\{0,\ldots,L{-}1\}$
  & $(t{-}\ell)\!\to\!t$: realization $x_{t-\ell}$ caused $x_t$ in \(s\)\\
$\mathcal{G}^s_{\mathcal{XY}}$
  & Sample-time\newline event-to-outcome graph
  & Time indices $\{0,\ldots,L{-}1\}$ $\cup$ outcomes $\mathcal{Y}_s$
  & $t\!\to\!Y_j$: realization \(x_t\) caused \(Y_j\) in $s$ \\
$\mathcal{G}^s = \mathcal{G}^s_{\mathcal{XX}}\cup\mathcal{G}^s_{\mathcal{XY}}$
  & Sample-time\newline causal graph
  & Time indices $\{0,\ldots,L{-}1\}$ $\cup$ outcomes $\mathcal{Y}_s$
  & Joint DAG over all events and outcomes in $s$ \\
\midrule
$\bar{\mathcal{G}} = \pi(\mathcal{G}^s)$
  & Sample-summary\newline graph
  & Event types $\mathcal{X}_s \subseteq \mathcal{X}$, outcomes $\mathcal{Y}_s \subseteq \mathcal{Y}$
  & $u\!\to\!v$: type $u$ caused type $v$ in $s$ \\
\midrule
$\mathcal{G}$
  & Population\newline summary graph
  & Event types $\mathcal{X}$, outcomes $\mathcal{Y}$
  & $u\!\to\!v$: type $u$ causes type $v$ across $\mathcal{D}$ \\
\bottomrule
\end{tabular}
\end{table}

\subsection{Definitions}

\begin{definition}[Markov Boundary]\label{def:mb} \citet{optimal_feature_set_cd}.
In a faithful BN \(<\boldsymbol{U}, \mathcal{G}, P>\), for a set of variables \(\boldsymbol{Z} \subset \boldsymbol{U}\) and label \(Y \in \boldsymbol{U}\), if all other variables \(X \in \{\boldsymbol{X} - \boldsymbol{Z}\}\) are independent of \(Y\) conditioned on \(\boldsymbol{Z}\) and any proper subset of \(\boldsymbol{Z}\) do not satisfy the condition, then \(\boldsymbol{Z}\) is the Markov Boundary of \(Y\): \(\textbf{MB}(Y)\).
\end{definition}

\subsection{Summary Graph}
Operators are often confronted with observing a single sequence during inference to understand generalizable rules (e.g., "Fire causes Smoke") rather than specific timestamps \((\mathcal{G}^s)\). We adapt the terminology from~\cite{assaad_survey_ijcai_cd_time_series} regarding summary causal graph (SCG), which treats the unique event types in \(\mathcal{X}\) as the variables of interest. Therefore, for a single sequence \(s\), we aim to recover its \emph{Summary Causal Graph}.
\begin{definition}[Sample-Summary Graph]\label{def:sample_summary_graph}
    Let \(\mathcal{X}_s \subseteq \mathcal{X}\) be the set of unique event types in \(s\). The Sample-Summary Graph \(\bar{\mathcal{G}} = (\mathcal{X}_s, \mathcal{E}_s)\) is the surjective projection of the Sample-Time Causal Graph \(\mathcal{G}^s\) onto \(\mathcal{X}_s\). Specifically, a type-level edge \(u \to v\) exists in \(\bar{\mathcal{G}}\) if and only if it appears at least once in the time graph:
    \begin{align}
        u \to v \in \mathcal{E}_s \iff \exists t, k \text{ s.t. } ((t-k) \to t) \in \notag\\ 
        \mathcal{E}_{t} \land x_{t-k} = u \land x_t = v \notag
    \end{align}
\end{definition}

\section{Background and Problem Formulation}

\paragraph{Data-Generating Process.}

We model the joint generation of events and outcomes as a stochastic process over a shared time index $t \in \mathbb{N}$ (Appendix, Fig.~\ref{fig:sample_level_graph}). Notation is described in Appendix~\ref{appendix:notations}.

Let $\{X_t\}_{t \geq 0}$ be a stochastic process taking values in 
a finite discrete alphabet $\mathcal{X}$. 
We make no finite-order Markov assumption; the full conditional \(P(X_t|X_{<t})\) may depend on the entire history — a great departure from standard methods~\cite{cd_temporaldata_review}.
A realization of length $L+1$ is denoted 
$s = (x_0, \ldots, x_L) \in \mathcal{X}^{L+1}$. Let $\mathcal{Y} = \{Y_1, \ldots, Y_c\}$ be a set of discrete outcome variables. Each binary outcome \(Y_j\) is generated conditionally on the event history as \(P(Y_j \mid X_{<t}) \) for \(c\) outcomes. Critically, the event chain Eq.~\eqref{eq:joint_dgp} is 
autonomous: it does not depend on $\mathcal{Y}$. We assume that outcomes are solely driven by the events (no inter-outcome effects~\cite{learningcommoncausalvarlabel}). The joint distribution therefore factorizes as:
\begin{equation}\label{eq:joint_dgp}
    P\!\left(X_{0:L},\, Y_{1:n}\right)
    \;=\;
    \prod_{t=0}^{L} P(X_t \mid X_{<t})\prod_{j=1}^{c} P\!\left(Y_j \mid X_{<L}\right),
\end{equation}

\begin{remark}[Two-chain DGP and Markovianity]\label{rem:imbrication}
The two-chain structure defines a natural edge partition. Intra-chain edges 
$X_{t-\ell} \to X_t$ describe the event-to-event time causal graph $\mathcal{G}^s_{\mathcal{X}\mathcal{X}}$; cross-chain edges 
$X_t \to Y_j$ form the event-to-outcome subgraph 
$\mathcal{G}^s_{\mathcal{X}\mathcal{Y}}$. 
No edges of the form $Y_j \to X_t$ or $Y_j \to Y_k$ exist, 
thus the joint sample-level graph \(\mathcal{G}^s\) is a DAG by construction. The sample-level regime of Fig.~\ref{fig:sample_level_graph} corresponds exactly to recovering subgraphs at different levels 
of aggregation. 
\end{remark}

\paragraph{Sample-Level Causal Graphs.}
In several domains, a practitioner has access to a single sequence at inference time. Population-level aggregation is either unavailable or
scientifically undesirable, as the causal structure of this sequence not
the average across a cohort, \emph{is the object of interest}. We define the causal graph types following the literature~\cite{assaad_survey_ijcai_cd_time_series}:

\begin{definition}[Sample-Time Causal Graph]\label{def:stcg}
Let $s$ be a realization of the event chain in Eq.~\eqref{eq:joint_dgp}. 
The Sample-Time Causal Graph
$\mathcal{G}^s_{\mathcal{X}\mathcal{X}} = (\mathcal{T},\, \mathcal{E}_t)$ is a DAG 
where nodes $\mathcal{T} = \{0, \ldots, L\}$ are time steps. 
A directed edge $(t-\ell) \to t$ exists in $\mathcal{E}_t$ if and 
only if the realization of the event at $t-\ell$ is a direct cause 
of the event at $t$.
\(\mathcal{G}^s_{\mathcal{XX}}\) answers: did $Fire@t_1 \to Smoke@t_2$ at the sample level. 
\end{definition}

\begin{definition}[Sample-Markov Boundary Graph]
\label{def:sample_mb_graph}
The Sample Markov Boundary Graph $\mathcal{G}^s_{\mathcal{XY}} = 
(\mathcal{T} \cup \mathcal{Y},\, \mathcal{E}_y)$ is a bipartite DAG. 
Edge $t \to Y_j$ exists in $\mathcal{E}_y$ if and only if the event 
realization $x_t$ causes $Y_j$. Under 
Eq.~\eqref{eq:joint_dgp}, $Y_j$ has neither children nor spouses, 
so $\mathcal{MB}(Y_j) = \mathrm{Pa}_{\mathcal{G}^s_{\mathcal{XY}}}(Y_j)$.
\end{definition}

\section{The \textsc{seq2cause} Framework}
\label{sec:framework}
\subsection{Neural Autoregressive Density Estimation}
\label{sec:arde}
To learn the described DGP we employ two AR models~\cite{math2024harnessingeventsensorydata} namely the event and outcome models $\mathrm{f}_{\theta_x}, \mathrm{f}_{\theta_y}$ respectively with parameters \(\theta_x, \theta_y\)\footnote{\(\mathrm{f}_{\theta_y}\) is stacked on top of the event model which serves as a backbone containing the majority of the parameters.}. They are trained once via cross-entropy minimization. They both map the trajectory \(x_{<t}\) to a distribution over the next event type and outcomes:

\begin{align}\label{eq:ar_event_head}
    P_{\theta_x}(X_t \mid X_{<t})
    &=\; \mathrm{Softmax}(\mathrm{f}_{\theta_x}(x_{<t})) \\
     P_{\theta_y}(Y_j\mid X_{<t})
    &= [\mathrm{Sigmoid}\!\bigl(\mathrm{f}_{\theta_y}(\mathrm{f}_{\theta_x}(x_{<t})))\bigr)]_j,
    \quad j = 1,\ldots,c.
\end{align}
Crucially, the likelihood of \(Y_j\) is estimated \emph{at every step $t$} even though it is only realized at $t = L$; this forces the model to recover the outcome likelihood from partial histories \cite{math2024harnessingeventsensorydata}, which is precisely the sequential conditional probabilities required by the described CI-test in \S\ref{sec:estimation}. 

\subsection{Assumptions and $\epsilon$-Oracle Model}
\label{sec:assumptions}

\textsc{seq2cause} assumes \emph{no parametric form for the DGP}, above the expressivity of the AR estimator. Beyond temporal precedence 
(A\ref{ass:temporal}) and causal sufficiency (A\ref{ass:sufficiency}), 
two assumptions distinguish our framework from prior work~\cite{assaad_survey_ijcai_cd_time_series}:

\begin{assumption}[$\epsilon$-Oracle Model]\label{ass:oracle}
The AR model $\mathrm{f}_\theta$ trained by maximum likelihood  
on data from $P$ satisfies:
\begin{equation}\label{eq:oracle_kl}
    D_{\mathrm{KL}}\!\bigl(P(X_t \mid X_{<t}) \,\|\,
    P_\theta(X_t \mid X_{<t})\bigr) \leq \epsilon \quad \forall\, t.
\end{equation}
\end{assumption}

Assumption~\ref{ass:oracle} does \emph{not} presuppose a perfect model: it states that whatever quality the backbone achieves, the causal identification error is bounded by that same quantity (Thm.~\ref{thm:cmi_error_bound}, Cor.~\ref{cor:adaptive_threshold}). 
Empirically, Fig.~\ref{fig:robustness} confirms that structural recovery succeeds well before full convergence ($\epsilon \approx 0.1$ suffices). No oracle is assumed; only that the practitioner can \emph{monitor} the gap (Appendix~\ref{app:epsilon_monitor}). Secondly, standard faithfulness~\cite{constrainct_based_cd} requires CI-tests 
to return exact zeros --- unrealistic under
imperfect density estimation. We instead adopt strong 
faithfulness~\cite{uhler2013geometry}: valid causal associations 
must exceed a significance threshold $\tau$. 
We provide a robustness analysis to assumption violation in Appendix~\ref{appendix:limitation}.

\subsection{Sample-Level: The Universal CMI Primitive}
\label{sec:estimator}

Both the \(\mathcal{X} \to \mathcal{X}\) and \(\mathcal{X} \to \mathcal{Y}\) regimes reduce to a 
single statistical primitive: testing whether a past event carries information about a target variable (event or outcome), given the 
observed trajectory. Let \(T_{t+1} \;\in \;  \bigl\{E_{t+1}\bigr\}\;\cup\; 
    \bigl\{Y_1, \ldots, Y_c\bigr\}, \) the target variable where $E_{t+1} \triangleq \mathbf{1}[{X_{t+1} = x_{t+1}]}$ is the 
binary event indicator and $Y_j$ is the $j$-th outcome variable. Nodes in \(\mathcal{G}^s\) are indexed by these binary random variables; an edge \(E_{t-\ell} \to E_t\) denotes a direct causal effect between the corresponding event realizations\footnote{\(E_t\) projects the multinomial sequence onto a binary sequence in which causal dependencies between events realization are well-defined in the sample-level graph \(\mathcal{G}^s\): \(E_0 \to E_4\) means that the occurrence of \(x_0\) at \(t_0\) caused event \(x_4\) at \(t_4\).}.
\(T_{t+1}\) is a binary variable, either the indicator of the next event type or outcome.

\paragraph{Conditional Mutual Information.}
\label{sec:cmi_unified}

In a sequence, we measure the remaining uncertainty of a target \(T_{t+1}\) when observing event \(E_t\) given $X_{<t} \triangleq \{X_0, \cdots, X_{t-1}\}$:
\begin{equation}\label{eq:cmi_unified}
    I(T_{t+1},\, E_t \mid X_{<t}) 
    \;\triangleq\; 
    H(T_{t+1} \mid X_{<t}) \;-\; H(T_{t+1} \mid E_t,\, X_{<t})
    \;=\; 
    \mathbb{E}_{e_t, x_{<t}}\!\left[
        I_G(T_{t+1},\, e_t \mid x_{<t})
    \right],
\end{equation}
where $I_G(T, e_t \mid x_{<t}) = D_{\mathrm{KL}}\!\bigl(
P(T \mid e_t, x_{<t}) \,\|\, P(T \mid x_{<t})\bigr)$ 
is the information gain~\cite{quinlan:induction, math2025oneshot} and \(I\) is the conditional mutual information (CMI) \cite{cover1999elements}. The unified CI-test follows directly:
\begin{equation}\label{eq:ci_test_unified}
    T_{t+1} \not\perp E_t \mid X_{<t} 
    \;\iff\; 
    I(T_{t+1},\, E_t \mid X_{<t}) > 0.
\end{equation}
Crucially, Eq.~\eqref{eq:cmi_unified} is identical for events or outcomes. The only regime-specific element is which 
conditional distribution $P_\theta(T_{t+1} \mid \cdot)$ is queried by the event or outcome head from the AR model.

\begin{remark}[Full-History Conditioning and Memory]\label{remark:full_history_conditioning}
We condition on the full trajectory $X_{<t}$ rather than the 
coarsened binary history $E_{<t}$. Under causal sufficiency~(A\ref{ass:sufficiency}), this blocks a potential backdoor 
path that would otherwise remain hidden in the binary projection. The \textit{sparse variant} (Appendix~\ref{app:sparse_aprox}) truncates \(X_{<t}\) to an horizon \(m\) to reduce memory consumption in long sequences for \(\mathcal{X} \to \mathcal{X}\). Empirically, the F1 score > \(80\%\) persists up to \(m=20\) (see Appendix, Fig.~\ref{fig:main_results}(c)).
\end{remark}

\paragraph{Lagged Effects \(\mathcal{X}\to\mathcal{X}\).}
Notably, for lagged effects $(t-\ell) \to t$ between events, the mediators \(M = X_{t-\ell+1:t-1}\) open an indirect path from \(E_{t-\ell}\) to \(E_t\); we randomize \(M\) to close it, which combined with temporal precedence (A\ref{ass:temporal}) and causal sufficiency (A\ref{ass:sufficiency}) making \(X_{<t-\ell}\) a valid backdoor adjustment set. We thus identify the existence of a direct edge \(E_{t-\ell} \to E_t\) as a randomized Controlled Direct Effect~\cite{pearl_2009}. For the \(\mathcal{X}\to\mathcal{Y}\) regime this issue does not arise: by the two-chain DGP~(Eq.~\ref{eq:joint_dgp}), \(Y_j\) has no children or spouses, so \(X_{<t}\) is already a valid adjustment set and no intervention is needed\footnote{We provide the full derivation, implementation, and parallelization in Appendix~\ref{appendix:algorithm_parallelization}, and use \(T_{t+1}\) throughout for clarity.}. \paragraph{Estimation and Error Bound.}\label{sec:estimation}
Since the true distribution $P$ is unknown, we use an 
$\epsilon$-oracle AR model (A\ref{ass:oracle}) 
to simulate $N$ i.i.d.\ history particles 
$\{x^{(l)}_{<t}\}_{l=1}^N \sim P_{\theta_x}(X_{<t})$ autoregressively via next-token prediction. Hence, the Monte Carlo estimator of Eq. \eqref{eq:cmi_unified} is given by:
\begin{equation}\label{eq:cmi_estimator}
    \hat{I}_N(T_{t+1},\, E_{t} \mid X_{<t})
    \;=\; 
    \frac{1}{N} \sum_{l=1}^N 
        \mathbb{E}_{e_t \sim P_\theta}\!\left[
            I_G\!\left(T_{t+1},\, e_t \;\middle|\; x^{(l)}_{<t}\right)
        \right].
\end{equation}
This estimator is consistent (Prop.~\ref{prop:consistency_cmi}) and requires only two queries to the model 
per particle: $P_\theta(T_{t+1} \mid x^{(l)}_{<t})$ and 
$P_\theta(T_{t+1} \mid e_t,\, x^{(l)}_{<t})$ (i.e., \textit{without} and \textit{with} cause \(E_t\)). Importantly, it is 
fully parallelizable on GPUs across positions $t$ and particles $l$. Since \(\hat{I}_N\) is an approximation and estimation of the true CMI \(I\), we now show that it is bounded by the model approximation error \(\epsilon\)  (Thm.~\ref{thm:cmi_error_bound}).
\subsection{Sample-Level Soundness}
\begin{theorem}[CMI Error Bound]\label{thm:cmi_error_bound}
Let $\hat{I}_N$ be the Monte Carlo estimator~(Eq. \ref{eq:cmi_estimator}) and $I$ the true CMI for the binary target $T_{t+1}$. Under the $\epsilon$-oracle assumption (A\ref{ass:oracle}) the error is bounded by:
\begin{equation}\label{eq:cmi_bound}
\limsup_{N \to \infty} \big|\hat{I}_N - I\big| \;\leq\; 2\,\Phi\!\Big(\!\sqrt{\epsilon_T\,/\,2}\Big) \;\leq\; 2\,\Phi\!\Big(\!\sqrt{\epsilon/2}\Big),
\end{equation}
where $\Phi(\delta) := \delta \ln 2 + (1{+}\delta)\,h_b\!\big(\tfrac{\delta}{1+\delta}\big)$, $h_b$ is the binary entropy, and \(\epsilon_T\)
is the binary approximation error for binary target $T$. The right inequality is the computable worst-case bound, monitored via the cross entropy loss. The left inequality is the operative bound: since each CI test queries a binary projection of the full categorical distribution, $\epsilon_T$ absorbs only the fraction of $\epsilon$ relevant to target $T$. 
\end{theorem}

\textit{(Proof Sketch)} We use the Alicki-Fannes-Winter \cite{Winter_2016} inequality for the difference between conditional entropies of two distributions with a small total variation distance \(TV(P, P_\theta) \leq 1/2\). 

Using Thm.~\ref{thm:cmi_error_bound}, as $\epsilon \to 0$, the CMI error vanishes and structural recovery becomes exact using standard faithfulness or must exceed the noise floor induced by \(\epsilon_T\). Even for finite \(N\), it is empirically dominated by the approximation bias induced by \(\epsilon_T\) for \(N>64\) (Appendix, Fig.~\ref{fig:ablation_n_particles}).

\begin{theorem}[Sample-Level Soundness]
\label{thm:soudness_sample_level}
Under \(\epsilon\)-strong faithfulness~(Def. \ref{def:strong_faithfulness}), causal sufficiency (A\ref{ass:sufficiency}), and temporal precedence 
(A\ref{ass:temporal}), \textsc{seq2cause} recovers the correct 
Sample-Time Causal Graph $\mathcal{G}^s$ as $N \to \infty$: 
an edge $(t{-}\ell) \to t$ is included iff 
$\hat{I}_N > 2\Phi(\sqrt{\epsilon_T/2})$ 
and excluded iff $\hat{I}_N \leq 2\Phi(\sqrt{\epsilon_T/2})$, 
with no false positives in the limit $N\to \infty$ and $\epsilon \to 0$.
\end{theorem}
\textit{(Proof Sketch)} By induction, we show that for each sequential step \(t\), we can recover the potential causes \(E_{<t}\) of the effect event \(E_t\) or \(Y_j\). By temporal precedence (A\ref{ass:temporal}) and causal sufficiency (A\ref{ass:sufficiency}), we can conclude that no other events affect the effect and thus verify the heredity.

\begin{corollary}[Root-Cause Analysis as Graph Traversal]
\label{cor:rca}
Under the conditions of Thm.~\ref{thm:soudness_sample_level}, 
\textsc{seq2cause} recovers the complete joint graph 
$\mathcal{G}^s = \mathcal{G}^s_{\mathcal{XX}} \cup 
\mathcal{G}^s_{\mathcal{XY}}$ from a single sequence~$s$. 
The root cause of outcome~$Y_j$ is then identifiable as the 
event realization~$x_{t^*}$ at the earliest time index 
admitting a directed path to the Markov boundary of~$Y_j$:
\begin{equation}
    \mathrm{RC}(Y_j,\, s) 
    \;=\; x_{t^*}, 
    \qquad
    t^* = \argmin_{t}\,\Bigl\{\,
        t \;\Big|\; 
        \exists\, t' \in \mathrm{MB}(Y_j),\;\,
        t \to t' 
        \text{ in } \mathcal{G}^s_{\mathcal{XX}}
    \Bigr\},
\end{equation}
This is the first identifiability result for sample-level 
root-cause analysis in discrete event sequences without 
topology assumptions or interventional 
data~\cite{assaad_rca_time_series_pmlr_2023, 
online_rca_kdd_2023, chain_of_event_2024_sre, han2025root} 
under the listed assumptions.
\end{corollary}

\subsection{Population-Level Discovery via Minimum Description Length}
\label{sec:population}
Population-level causal discovery seeks the global population summary graph $\mathcal{G} = 
(\mathcal{X} \cup \mathcal{Y},\, \mathcal{E})$ given $n$ i.i.d.\ sequences $\{s^{(k)}\}_{k=1}^n$. A directed 
edge $u \to v$ encodes the \emph{type-level} causal claim: 
``\texttt{Fire} causes \texttt{Smoke} across the population''
~\cite{math2025towards, assaad_survey_ijcai_cd_time_series}, as opposed to the 
sample-level claim ``\texttt{Fire}@$t_1$ caused 
\texttt{Smoke}@$t_2$'' recovered by the single sample time graph~\cite{math2025oneshot, tf_causalinterpretation_neurips_2023}.
\paragraph{Two-Part Code.}
To aggregate causal evidences, we formulate a minimum description length (MDL)~\cite{grunwald2004tutorialintroductionminimumdescription} criterion that operates directly on the summary graph $\mathcal{G}$ over event types and outcomes. Formally, the best graph is the one that minimized the cost \(\mathcal{L}(\mathcal{D},\mathcal{G})\)~\cite{grunwald2004tutorialintroductionminimumdescription, janzing_amc} as: 
\begin{equation}\label{eq:vanilla_mdl}
    \mathcal{L}(\mathcal{D}, \mathcal{G}) = \underbrace{\mathcal{L}(\mathcal{D|\mathcal{G}})}_{\text{data encoding cost}} + \underbrace{\mathcal{L}(\mathcal{G})}_{\text{model cost}}
\end{equation}
 
Each type-level edge $u \to v$ encodes a source $u \in \mathcal{X}$ 
and target $v \in \mathcal{X} \cup \mathcal{Y}$ at a cost of \(c_{\mathrm{type}} = \log|\mathcal{X}| 
    + \log(|\mathcal{X}|+|\mathcal{Y}|)\) nats by Rissanen's prefix-free integer code~\cite{rissanen1978modeling}. Then, the data encoding cost is the negative log likelihood (NLL) of the dataset under $\mathcal{G}$~\cite{cover1999elements}, well-defined through the 
acyclicity of the sample-level factorization 
(Eq.~\ref{eq:joint_dgp}). Therefore, adding edge $u \to v$ changes $\mathcal{L}(\mathcal{D}, \mathcal{G})$:
\begin{small}
\begin{equation}\label{eq:mdl_delta}
    \Delta\mathcal{L}(\mathcal{D}, \mathcal{G})
    = 
      \sum_{(k, t,\ell)\,\in\,\mathcal{M}_{uv}}
      \log\frac{
          P\!\bigl(T^{(k)}_t = 1\mid 
                \mathrm{Pa}_{\mathcal{G}}(T^{(k)}_t) \cup 
                \{E^{(k)}_{t-\ell}\}\bigr)
      }{
          P\!\bigl(T^{(k)}_t = 1\mid 
                 \mathrm{Pa}_{\mathcal{G}}(T^{(k)}_t) \bigr)
      } - c_{\mathrm{type}} 
\end{equation}
\end{small} 
where $\mathcal{M}_{uv} := \{(k, t,\ell) : x_{t-\ell}^{(k)} = u, x^{(k)}_t \;\text{or}\; y_j^{(k)} = v\}$ are the matching positions of source type $u$ at lag $\ell$ and target \(v\) at \(t\) in sequence $k$, $\mathrm{Pa}_{\mathcal{G}}(T^{(k)}_t)$ the current parent set of $T_t$ in $\mathcal{G}$, and $E_{t-\ell} = \mathbf{1}[{X_{t-\ell} = u]}$ the 
candidate cause at lag $\ell$. Crucially, the edge \(u\to v\) \emph{is included if $\Delta\mathcal{L} > 0$} (the 
accumulated likelihood gain exceeds the model cost).

\paragraph{Generative Error Correction for MDL Inclusion.}
Since $P$ is replaced by $P_\theta$ (A\ref{ass:oracle}), 
the approximated log-ratio in Eq.~\ref{eq:mdl_delta} deviates at every matched position. It therefore accumulates across sequences and degrades precision as the number of samples \(n\) grows. 
The following Lem.~\ref{lem:mdl_deviation} characterizes this deviation 
in function of \(\epsilon\) and yields the self-corrected inclusion criterion of 
Cor.~\ref{cor:adaptive_threshold}.
\begin{lemma}[Data Encoding Cost Deviation under the $\epsilon$-Oracle]
\label{lem:mdl_deviation}
Let $\Delta L^P$ and $\Delta L^\theta$ denote the true and 
approximated MDL gains for edge $u \to v$ 
(Eq.~\ref{eq:mdl_delta}) under $P$ and $P_\theta$ respectively.
Under A\ref{ass:oracle}, the accumulated deviation satisfies:
\begin{equation}
    \left|\Delta \mathcal{L}^\theta - \Delta \mathcal{L}\right|
    \;\leq\; \sqrt{\epsilon_T/2}\;\cdot\;\eta_{uv}
\end{equation}
\end{lemma}

where $\epsilon_T \leq \epsilon$ is the local binary 
approximation error (Prop.~\ref{prop:binary_projection}), 
and $\eta_{uv}$ is fully computable from $P_\theta$ and 
the data (Appendix~\ref{proof:threshold_mdl}, Eq.~\ref{eq:compute_threshold}). Concretely, $\eta_{uv} = \sum_{(k,t) \in \mathcal{M}_{uv}} \kappa_t^{(k)}$ accumulates per-position inverse-margin terms $\kappa_t^{(k)} \propto 1/(q_\pm - \sqrt{\epsilon_T/2})$, where $q_\pm = P_\theta(T_t \mid Z_\pm)$ are the \emph{same} two AR forward passes already evaluated for $\Delta\mathcal{L}^\theta$ --- so $\eta_{uv}$ requires no additional model queries.


\begin{corollary}[Adaptive Population Threshold]
\label{cor:adaptive_threshold}
Using Lem.~\ref{lem:mdl_deviation}, edge $u \to v$ is included 
in the population summary graph $\mathcal{G}$ if and only if:
\begin{equation}\label{eq:adaptive_threshold}
\boxed{
    u \to v \in \mathcal{G}
    \;\iff\;
    \Delta \mathcal{L}^\theta  \;>\; \tau_{uv}^*,
    \qquad
    \tau_{uv}^* 
    = \underbrace{c_{\mathrm{type}}}_{\text{MDL penalty}}
    + \;\;\; \underbrace{ \sqrt{\epsilon_T/2}\cdot\eta_{uv}}_{\text{error correction}}
}
\end{equation}
\end{corollary}
The threshold $\tau_{uv}^*$ is adaptive in two compounding senses. The MDL penalty $c_{\mathrm{type}}$ is a fixed one-time encoding cost, penalizing weak evidence for adding \(u \to v\). And the error correction $\sqrt{\epsilon_T/2}\cdot\eta_{uv}$ is 
pair-specific, vanishing as $\epsilon_T \to 0$ such that a better oracle requires no manual threshold recalibration. 
\begin{remark}[Computation of $\tau^*_{uv}$]
\label{rem:tau_computation}
Importantly, $\tau_{uv}^*$ requires no Monte-Carlo estimation, yielding memory complexity 
$\mathcal{O}(B \times M \times L^2)$ for the GPU staircase 
implementation (Appendix, Fig.~\ref{fig:trace_diagram}) for \(\mathcal{X} \to \mathcal{X}\) and just $\mathcal{O}(B \times L)$ for \(\mathcal{X} \to \mathcal{Y}\), where $B$ is the 
batch size and $M$ mediators to sample.
\end{remark}

\section{Experiments}
We evaluate \textsc{seq2cause} across all four regimes of
Tab.~\ref{tab:framework}, unifying results from synthetic non-linear SCMs and real-world vehicle diagnostic logs. Our evaluation always proceeds in two phases (1) reuse or train a pretrained AR model (2) apply \textsc{seq2cause}. Full evaluation protocol, robustness analysis and ablations are further discussed in Appendix~\ref{appendix:evaluation}. Sample-Level experiments are standardized under 10 runs using cross-validation, since population-level is an average, we report directly the mean.
\subsection{Setup}
\label{sec:setup}
 We recall that methods like PCMCI~\cite{pcmci}, DYNOTEARS~\cite{dynotears}, VARLiNGAM~\cite{varlingam}, RHINO~\cite{gong2023rhino}, GOLEM~\cite{golem},  TCDF~\cite{tcdf}, CUTS~\cite{yuxiao2023cuts}, NOTEARS~\cite{dag_no_tears}, DAG-GNN~\cite{neural_dag}) all operate on 
\emph{multivariate continuous time series or tabular data}, not on discrete sequences of events.

 
\paragraph{Nonlinear-SCMs (\textit{\(\mathcal{X} \to \mathcal{X}\), sample and population)}.} We validate \textsc{seq2cause} on sequences generated by nonlinear-SCMs with controllable memory $m$, sequence length \(L\), and vocabulary size $|\mathcal{X}|$. We train a LLaMa model~\cite{touvron2023llamaopenefficientfoundation} on the SCMs and validate the trainings by monitoring the normalized oracle scores \(\hat{\epsilon}\) (Eq.~\ref{eq:oracle_score}). We then apply \textsc{seq2cause} to recover the summary graphs. For the sample-level, we compare it against Neural Granger
(same AR backbone, probability difference instead of CMI),
Attention-BERT~\cite{bert},
Saliency (Input$\times$Gradient)~\cite{inputxgradient},
Shapley Value Sampling~\cite{shap},
and naive Random / Frequency baselines. For the population-level, we compare against Hawkes Processes (SHTP)~\cite{shtp}, CASCADE~\cite{cueppers2024causal} and light co-occurences baselines such as point-wise mutual information (PMI)~\cite{church-hanks-1990-word} and TF-IDF ~\cite{tfidf} baselines using individual significance threshold.


\paragraph{Vehicle Diagnostic Sequences (\textit{\(\mathcal{X} \to \mathcal{Y}\), sample and population}).}
We adopt a real-world vehicle dataset~\cite{math2024harnessingeventsensorydata}, comprising
$m{=}300{,}000$ sequences of error codes
over $|\mathcal{X}|{=}29{,}100$ event types and
$|\mathcal{Y}|{=}474$ failure outcomes
with $150$ events on average per sequence.
The two frozen AR backbones from~\cite{math2024harnessingeventsensorydata}
($105$M parameters with a cross entropy loss of \(1.91\) nats on the test set) serve as $\mathrm{f}_{\theta_x}$
and $\mathrm{f}_{\theta_y}$; neither saw the test set during training.
Ground truth MB were annotated by domain
experts; some outcomes contain partial annotations, making this a conservative evaluation. We compare against:
IAMB~\cite{iamb}, CMB~\cite{cmb},
MB-by-MB~\cite{mbb-by-mbb}, PCDbyPCD~\cite{pcdpcd},
and MI-MCF~\cite{mimcf} via the
\textit{PyCausalFS} package~\cite{causality_based_feature_selection_2019} and add a correlation method TF-IDF~\cite{tfidf}. We report weighted Precision, Recall, and F1 on the predicted set. An example of summary graph is shown in Fig.~\ref{fig:time_instance_graph_dtc}.



\begin{table}[!t]
\centering
\caption{%
  \textbf{Unified evaluation of \textsc{seq2cause} across all four causal regimes.}
  \textbf{(a)} $|\mathcal{X}|=1{,}000$, $L=64$, $\epsilon{=}0.05$,
  $\tau{=}3{\times}10^{-5}$, $N{=}128, H(P)=2.4$ nats.
  \textbf{(b, d)} Diagnostic vehicle sequences leading to critical failures as outcomes, $|\mathcal{X}|{=}29{,}100$,
  $|\mathcal{Y}|{=}474$, $L{\approx}150$, \(n=50, 000\) sequences, weighted avg, $\dagger$: timeout ${>}3$ days \textbf{(c)} $|\mathcal{X}|=1{,}000$, $L=64$, $\epsilon{=}0.05$, $n{=}50,000 \;\text{sequences}, H(P)=2.4$ nats.
}
\label{tab:unified}
\setlength{\tabcolsep}{4pt}
\renewcommand{\arraystretch}{1.2}
\resizebox{\linewidth}{!}{%
\begin{tabular}{
  @{} 
  r                              
  c                              
  ccc                            
  | c                            
  ccc                            
  @{}
}
 
\toprule
\rowcolor{headerblue}
& \multicolumn{4}{c|}{\textbf{Event-to-Event}\quad$\mathcal{X}\to\mathcal{X}$}
& \multicolumn{4}{c}{\textbf{Events-to-Outcomes}\quad$\mathcal{X}\to\mathcal{Y}$} \\
 
\midrule
\rowcolor{cellbg}
& \multicolumn{4}{c|}{\textbf{(a)} \textit{Sample-Summary Graph}
    $\;\mathcal{G}^s_{\mathcal{XX}}$  (nonlinear-SCM)}
& \multicolumn{4}{c}{\textbf{(b)} \textit{Sample-Markov Boundary}
    $\;\mathcal{G}^s_{\mathcal{XY}}$~(Vehicle Diagnostics)} \\
 
\multirow{6}{*}{\rotatebox{90}{\small\textbf{Sample-Level}}}
& \textit{Method} & SHD$\downarrow$ & F1$\uparrow$ & Prec.$\uparrow$
& \textit{Method} & Prec.$\uparrow$ & Rec.$\uparrow$ & F1$\uparrow$ \\
\cmidrule(lr){2-5}\cmidrule(lr){6-9}
& Attention (BERT)          & 321.0{\tiny$\pm$15}   & 50.1{\tiny$\pm$.01} & 35.0{\tiny$\pm$.01}
& IAMB~\cite{iamb}                      & \na & \na & \na \\
& Saliency (IG, LLaMA)~\cite{inputxgradient}      & 160.2{\tiny$\pm$6.6}  & 67.3{\tiny$\pm$.01} & 51.4{\tiny$\pm$.01}
& CMB~\cite{cmb}                       & \na & \na & \na \\
& Shapley (LLaMA)~\cite{shap}           & 148.0{\tiny$\pm$5.1}  & 60.0{\tiny$\pm$.01} & 55.5{\tiny$\pm$.01}
& MB-by-MB~\cite{mbb-by-mbb}                  & \na & \na & \na \\
& Neural Granger (LLaMA)~\cite{granger_causality}    & 100.2{\tiny$\pm$14.6} & 69.2{\tiny$\pm$.04} & 71.6{\tiny$\pm$.04}
& PCDbyPCD\cite{pcdpcd}                  & \na & \na & \na \\
\cmidrule(lr){2-5}\cmidrule(lr){6-9}
& \textbf{seq2cause}        
    & \textbf{28.6}{\tiny$\pm$2.8} 
    & \textbf{91.3}{\tiny$\pm$.01} 
    & \textbf{89.7}{\tiny$\pm$.01}
& \textbf{seq2cause}        
    & \textbf{58.2}{\tiny$\pm$1.4} 
    & \textbf{35.5}{\tiny$\pm$0.8} 
    & \textbf{46.1}{\tiny$\pm$1.0} \\
 
\midrule
\rowcolor{cellbg}
& \multicolumn{4}{c|}{\textbf{(c)} \textit{Population-Summary Graph}
    $\;\mathcal{G}_{\mathcal{XX}}$\;\textnormal{(nonlinear-SCM)}}
& \multicolumn{4}{c}{\textbf{(d)} \textit{Population-Markov Boundary}
    $\;\mathcal{G}_{\mathcal{XY}}$ (Vehicle Diagnostics)} \\
\multirow{6}{*}{\rotatebox{90}{\small\textbf{Population-Level}}}
& \textit{Method} & Prec.$\uparrow$ & Rec.$\uparrow$ &  F1$\uparrow$
& \textit{Method} & Prec.$\uparrow$ & Rec.$\uparrow$ & F1$\uparrow$ \\
\cmidrule(lr){2-5}\cmidrule(lr){6-9}
& PMI~\cite{church-hanks-1990-word} \((\tau = 1)\)         & 5.6 & 91.6 & 10.5
& IAMB~\cite{iamb}            & \na & \na & \na \\
& TF-IDF~\cite{tfidf} \((\tau=0.05)\) & 17.4 & 89.6 & 29.1 
& CMB~\cite{cmb}             & \na & \na & \na \\
& SHP~\cite{shtp} (Hawkes)   & \na & \na & \na 
& PCDbyPCD~\cite{pcdpcd}        & \na & \na & \na \\
& CASCADE~\cite{cueppers2024causal}         & \na & \na & \na
& MI-MCF~\cite{mimcf}          & \na & \na & \na \\
& \textsc{seq2cause} \textit{w/o corr.} & 4.2 & 99.7 & \textbf{8.0}
& TF-IDF~\cite{tfidf} & 41.3{\tiny$\pm$2.1} & 18.7{\tiny$\pm$1.4} & 25.8{\tiny$\pm$1.6} \\
\cmidrule(lr){2-5}\cmidrule(lr){6-9}
& \textbf{seq2cause} & \textbf{81.5} & \textbf{54.6} & \textbf{65.4}
& \textbf{seq2cause}     
    & \textbf{68.6}{\tiny$\pm$1.4} 
    & \textbf{48.0}{\tiny$\pm$1.6} 
    & \textbf{56.8}{\tiny$\pm$1.2} \\
\bottomrule
\end{tabular}
}
\end{table}

\subsection{Unified Four-Regime Evaluation}
\label{sec:unified} 

\paragraph{Cell (a) — Sample-level $\mathcal{X}\!\to\!\mathcal{X}$.}
\textsc{seq2cause} outperforms the strongest baseline
(Neural Granger) by over 20~F1 points,
achieving $91\%$ F1 and $89\%$ Precision at
$|\mathcal{X}|\!=\!1{,}000$.
Attention-based methods fail to separate correlation
from causation ($50\%$ F1 score); saliency and Shapley methods
offer partial signal ($67\%$ and $60\%$ F1 respectively)
but lack the interventional grounding of CMI testing.
This confirms that \textbf{measuring the CMI is a
fundamentally more reliable causal signal} than
monitoring single-token probability fluctuations
or gradient-based attributions.

\paragraph{Cell (c) — Population-level $\mathcal{X}\!\to\!\mathcal{X}$.}
Hawkes-based methods (SHP) and CASCADE are architecturally 
excluded: co-occurrence degeneracy (Eq.~\ref{eq:co_occ}) renders 
their excitation matrices unidentifiable at $|\mathcal{X}|{=}1{,}000$, 
and both time out before $N{=}100$ sequences.
Correlation baselines (PMI, TF-IDF) terminate in time but 
conflate causation with association: PMI achieves Recall~91.6\% 
at Precision~5.6\% — effectively predicting all edges — while 
TF-IDF's heuristic threshold yields F1~29.1\%, with no mechanism 
to distinguish direct causes from spurious co-occurrences.
\textsc{seq2cause} \textit{without} the MDL error correction 
(Cor.~\ref{cor:adaptive_threshold}) mirrors this failure: 
Precision collapses to 4.2\% as oracle approximation errors 
accumulate unchecked beyond $N{\approx}1{,}000$ sequences, 
providing direct empirical validation of Lem.~\ref{lem:mdl_deviation}.
With correction, \textsc{seq2cause} achieves F1~65.4\% 
(Prec.~81.5\%, Rec.~54.6\%) — the only method to maintain 
both high precision and moderate recall simultaneously, and \textbf{the only 
causally grounded solution viable}. Association baselines are acutely sensitive to $\tau$: a shift of $0.02$ in TF-IDF's threshold alone can collapse F1 by over 20 points, with no principled criterion to guide that choice (Appendix~\ref{sec:baselines}).
\textsc{seq2cause}'s decision boundary is \textbf{derived analytically from the error bound and requires no
calibration}.
Full evolution across samples is shown in 
Appendix Fig.~\ref{fig:nb_seq_baseline}. We provide a deeper discussion regarding threshold calibration for the baselines in Appendix~\ref{sec:baselines}.
\paragraph{Cells (b, d) — Sample and Population $\mathcal{X}\!\to\!\mathcal{Y}$.}
Classical algorithms are architecturally excluded from this regime by two compounding degeneracies that no amount of compute can resolve.
\textit{Degeneracy 1:} vocabulary-driven intractability. At $|\mathcal{X}|=29,100$, the candidate conditioning-set space is $2^{|\mathcal{X}|}$, and even the greedy IAMB shell requires $\mathcal{O}(|\mathcal{X}|)$ CI tests per outcome, each conditioning on a growing set of discrete covariates whose joint support is never populated. All five baselines exceed a 3-day wall-clock budget on $n = 50,000$ sequences across 4 GPUs.
\textit{Degeneracy 2}: sample-driven impossibility at reduced vocabulary. To rule out "more compute would help", we evaluate on a stress-tested reduced setting in Appendix Tab.~\ref{tab:reduced_n_mb}: we restricted the vocabulary to the top-K most frequent event types ($K \in \{500, 1000, 2000\}$) and the sample count to $n \in \{500, 5000\}$. At $K = 500, n = 500$, classical MB algorithms do run to completion, but the expected co-occurrence count for any $(u, Y_j)$ pair 
tend to \(0\) — therefore any asymptotic CI test reject independence. The empirical result confirms this: Precision and Recall collapse to 0 across all five baselines (Tab.~\ref{tab:reduced_n_mb}), with CI tests returning "independent" on essentially all pairs.
Scaling up either axis closes one degeneracy and opens the other: more samples at fixed $K$ makes Deg. 2 less acute but blows up Deg. 1; more vocabulary at fixed $n$ re-opens Deg. 2. \textsc{seq2cause} escapes this trap because the AR backbone amortizes density estimation across all $(u, Y_j)$ pairs simultaneously, reusing the same representations for every CI test rather than re-estimating joint supports per pair.




\subsection{Generative Error and Scalability}
\label{sec:scalability}

 \paragraph{Identifiability precedes convergence.}
Figure~\ref{fig:robustness_and_scale}(a) shows that exact model
convergence ($\epsilon\!\to\!0$) is not necessary
for structural recovery.
A distinct phase transition occurs at $\epsilon\!\approx\!0.1$:
the coarse-grained causal graph is recovered before
the model masters fine-grained transition probabilities.
\textsc{seq2cause} exhibits a \textbf{conservative
failure mode} — as $\epsilon$ increases,
Recall degrades while Precision stays near 0.95,
meaning errors manifest as missed edges rather
than hallucinated ones.
This is a desirable property for safety-critical settings
where false positives are more costly than false negatives.
\begin{figure}[t]
\centering
\begin{subfigure}[t]{0.32\columnwidth}
    \centering
    \includegraphics[width=\linewidth, height=3.8cm, keepaspectratio=false]%
        {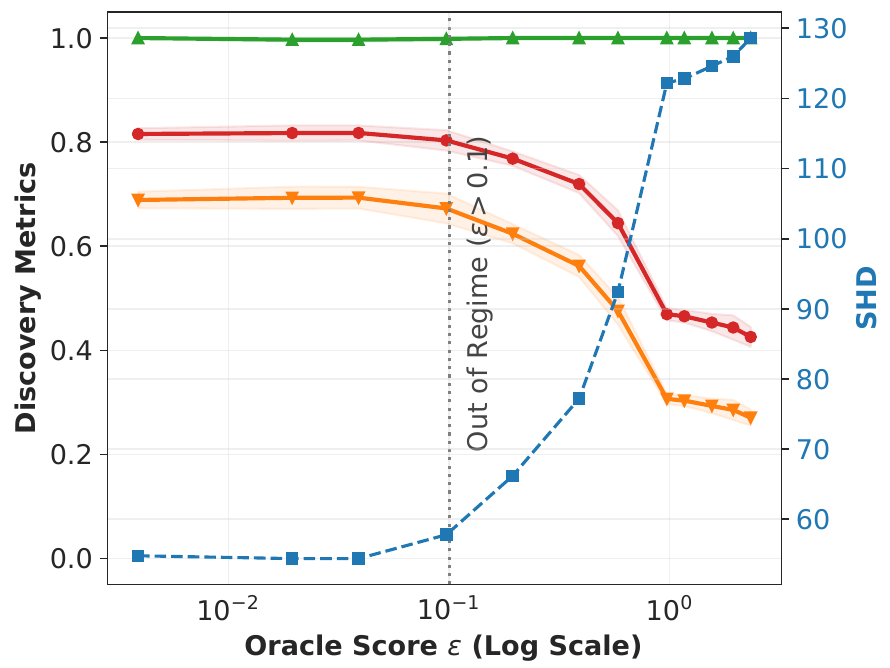}
    \caption{Sample-Level: Robustness to \(\epsilon\)}
    \label{fig:robustness}
\end{subfigure}
\hspace{0.02\columnwidth}
\begin{subfigure}[t]{0.65\columnwidth}
    \centering
    \includegraphics[width=\linewidth, height=3.8cm, keepaspectratio=false]%
        {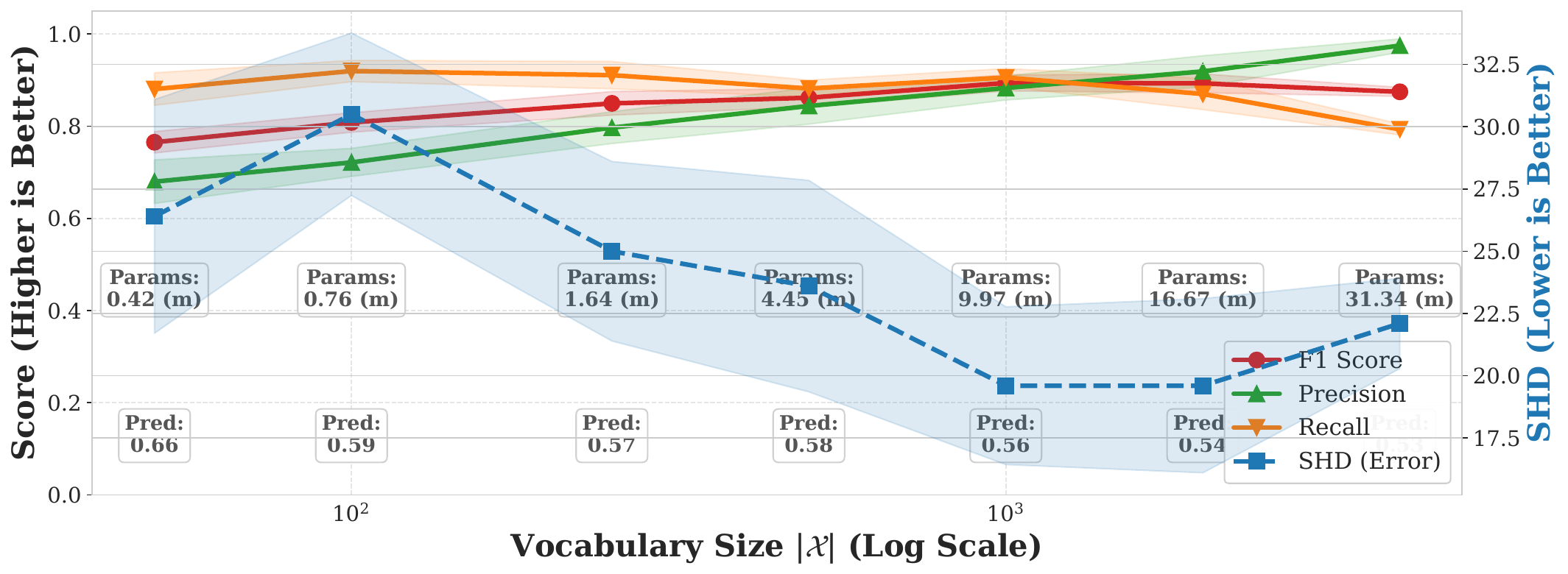}
    \caption{Sample-Level: Scalability to Vocabulary Size}
    \label{fig:vocabulary}
\end{subfigure}
\caption{%
  \textbf{Robustness and scalability of \textsc{seq2cause}.}
  \textbf{(a)}~Phase transition at $\epsilon{\approx}0.1$:
  Precision stays near 0.95 as $\epsilon$ increases
  (conservative failure mode).
  \textbf{(b)}~Stable performance across vocabularies from
  $|\mathcal{X}|{=}10^2$ to $8{\times}10^3$; Shannon Redundancy (\textit{Pred} = $1- H(P)/H_{max}$) accounts for modeling difficulty ($\hat{\epsilon}=0.01, L=64, N=64, \tau=10^{-4}$).%
}
\label{fig:robustness_and_scale}
\end{figure}

\paragraph{Breaking the curse of dimensionality.}
Figure~\ref{fig:robustness_and_scale}(b) evaluates \textsc{seq2cause}
across vocabularies from $|\mathcal{X}|\!=\!10$
to $8{,}000$.
F1 remains stable at ${\approx}81\%$ despite the
exponential growth of the candidate edge space,
confirming that the pretrained AR backbone effectively
amortizes the density estimation cost across all
event types.
No existing constraint-based or Hawkes-based method
can operate at $|\mathcal{X}|\!>\!100$ in this regime.
 
\textbf{Self-correcting MDL threshold $\tau^*_{uv}$ prevents
  precision collapse.} Tab.~\ref{tab:unified} and Appendix Fig.~\ref{fig:ablation_correction} confirms this 
empirically: without the correction, Precision collapses to 
near zero beyond $N{\approx}1{,}000$ sequences as approximation 
errors accumulate~(Lem.~\ref{lem:mdl_deviation}), while $\tau^*_{uv}$ maintains Precision 
above \(80\%\) across the full range $n \in [10,\, 20\cdot 10^4]$.
\section{Conclusion}
Empirically, \textsc{seq2cause} is the first method to populate 
all four regimes of Tab.~\ref{tab:framework} with a single frozen 
backbone: it achieves \(91\%\) F1 on sample-level event graphs where 
the strongest baseline reaches \(69 \%\). It recovers MB algorithms at \(|\mathcal{X}|=29,100\) intractable while recovering population-level boundaries in minutes. The MDL aggregation enables bypassing the noise when stacking thousands of sequences and is the most striking superiority over baselines, confirming Lem.~\ref{lem:mdl_deviation}. 
Thms.~\ref{thm:cmi_error_bound}, \ref{thm:soudness_sample_level} 
explain why: minimizing next-token cross-entropy simultaneously 
bounds CMI error, guarantees soundness, and self-calibrates the 
population threshold through the single monitorable quantity 
$\epsilon$, so every advance in sequence modeling tightens 
causal identification guarantees without retraining.

\bibliography{neurips_2026}


\appendix

\section{On the Limitation of Multivariate Representation for Event Sequences}
We detail the limitation of representing discrete event sequences as multivariate time series in the high-dimensional regime. 

\subsection{Multivariate Hawkes Process Log-likelihood Derivation}\label{derivation_ll_hawke_mutual}

\begin{definition}[Mutually-Exciting Hawkes processes]
    Consider \(\mathbf{N}(t) = (N^1(t), \cdots, N^{|\mathcal{X}|}(t))\) as a collection of \(|\mathcal{X}|\) counting processes with \(N^k(t)\)'s occurrence times denoted as \(t^k_1, t^k_2\) etc. They are mutually-exciting Hawkes processes if \(N^k(t)\)'s conditional intensity follows:
    \[\lambda^*_k(t) = \lambda + \sum^{|\mathcal{X}|}_{j=1}\sum_{t^j_i < t} \mu_{j,k}(t-t^j_i) \; \text{for }\; k=1, \cdots, |\mathcal{X}|\]
    where \(\mu_{j,k}(s) \geq 0\)
\end{definition}

The likelihood for any point process parametrized by \(\theta\) with observations \(\{t_1, \cdots, t_L\}\) within a time horizon \(0 \leq t \leq T\) can be computed as the sum of the log-likelihood for each process~\cite{hawke_process_and_app_2025}: 

\begin{equation*}
    \ln \mathcal{L} (\theta|t_1, \cdots, t_L, T) = \sum^{|\mathcal{X}|}_{k=1} \ln \mathcal{L}_k (\theta|t_1, \cdots t_L)
\end{equation*}
where each term is defined by:
\begin{align}
    \ln{\mathcal{L}_k}(t_1, \cdots, t_L) &= \ln \left(\left[\prod^L_{i=1} \lambda^{*}_k(t_i)\right]\exp{(-\int^T_0 \lambda^{*}_k(t) dt)}\right) \notag\\
\end{align}
This expression reduces to:
\begin{align}
    &= \sum^L_{i=1}\ln{\lambda^{*}_k(t_i)} + \ln{\left(\exp{(-\int^T_0 \lambda^{*}_k(t) dt)}\right)} \notag\\
    &= \sum^L_{i=1}\ln{\lambda^{*}_k(t_i)} -\int^T_0 \lambda^{*}_k(t) dt\label{eq:mll_mutual_hawkes_per_event}
\end{align}
Finally, we have for the total log-likelihood:
\begin{equation}\label{eq:mll_mutual_hawkes_appendix}
     \ln \mathcal{L} (\theta|t_1, \cdots, t_L, T) = \sum^{|\mathcal{X} |}_{k=1}\sum^L_{i=1}\ln{\lambda^{*}_k(t_i)} - \sum^{|\mathcal{X}|}_{k=1}\int^T_0 \lambda^{*}_k(t)dt
\end{equation}

\subsection{Limitations}

\paragraph{Number of parameters.}
The \(N^j_t\) process has either an excitatory effect on \(N^k_t(\phi_{j,k}>0)\) or no effect on \(N^k_t(\phi_{j,k} = 0)\) for \(j = k \) or \(j \not= k\). Importantly, these processes have \(\mathcal{O}(|\mathcal{X}|^2)\) parameters to fit, which is often intractable in practice for many applications~\cite{hawke_process_and_app_2025}. For \(|\mathcal{X}|= 10^3\) the number of parameters is already \(1\) million. Moreover, the actual number of scalars being optimized in practice is more than \(|\mathcal{X}|^2\), since there is the base intensity \(\lambda\), the parameters of the excitation functions \(\mu\) (decay rates) and interaction magnitude. The literature often assumes \(|\mathcal{X}|^2+|\mathcal{X}|\)~\cite{hawke_process_and_app_2025}. 

\paragraph{Optimization.}
To understand why the Hawkes process is insufficient for our setting, we first show that the MLE of its parameters is computationally intractable at scale. We therefore seek the parameters (base rate $\alpha$, decay $\beta$) that make the observed sequence of events most probable. The likelihood for any point process parametrized by \(\theta\) with observations \(\{t_1, \cdots, t_L\}\) within a time horizon \(0 \leq t \leq T\) can be computed as the sum of the log-likelihood for each process: 
\begin{equation}\label{eq:mll_mutual_hawkes}
     \ln \mathcal{L} (\theta|t_1, \cdots, t_L, T) = \sum^{|\mathcal{X} |}_{k=1}\sum^L_{i=1}\ln{\lambda^{*}_k(t_i)} - \sum^{|\mathcal{X}|}_{k=1}\int^T_0 \lambda^{*}_k(t)dt
\end{equation}
Importantly, Eq.~\ref{eq:mll_mutual_hawkes} cannot be computed easily since the sum over all event types \(|\mathcal{X}| \) and the double summation over all pairs of event \(t_i\) and \(t_j\), leading to \(\mathcal{O}(|\mathcal{X}|^2\cdot N(T)^2)\) at worst and \(\mathcal{O}(|\mathcal{X}|^2\cdot N(T))\) if using exponential decay~\cite{hawke_process_and_app_2025}. This is a regime where classical MLE is computationally prohibitive and statistically prone to overfitting~\cite{hawke_sparse_mutual, sparsetemporalattention}.

\section{Assumption}\label{appendix:assumption}

\subsection{Definitions}\label{appendix:assumptions:def}
\begin{assumption}[Temporal Precedence]
\label{ass:temporal}
Given a recorded sequence with monotonically increasing 
timestamps $0 \leq t_1 \leq \cdots \leq t_L$, an event $x_t$ 
may influence any subsequent event $x_{t'}$ only if $t < t'$.
Edge orientation in $\mathcal{G}$ is therefore determined by 
time order.
\end{assumption}

\begin{assumption}[Causal Sufficiency]
\label{ass:sufficiency}
All causally relevant variables are observed. 
There are no latent confounders affecting the events or the outcomes.
\end{assumption}

\begin{definition}[$\epsilon$-Strong Faithfulness]\label{def:strong_faithfulness}
Let $\tau_{\epsilon_T} = 2\Phi(\sqrt{\epsilon_T/2})$ be the asymptotic approximation error bound of the model (Thm.~\ref{thm:cmi_error_bound}). A distribution $P$ is \textbf{$\epsilon$-Strong Faithful} to a causal graph $\mathcal{G}$ with respect to the estimator $P_\theta$ if, for every active edge $E_{t-\ell} \to T_t \; $ the \textbf{true} CMI satisfies:
\begin{equation}
I(T_t; E_{t-\ell} \mid X_{<t}) > 2\cdot \tau_{\epsilon_T}
\end{equation}
\end{definition}
\begin{example}[$\epsilon$-Strong Faithful CI-test at two marginal frequencies]
Fix $|\mathcal{X}| = 1{,}000$, $\epsilon = 0.05$ (a well-trained backbone,
cf.\ Table~\ref{fig:main_results}), and consider a candidate cause 
$E_{t-\ell}$ whose target $T_t = E_{t+1}$ has marginal frequency
$p_v = \mathbb{P}(X_{t+1} = v)$.
By Prop.~\ref{prop:binary_projection}, the binary approximation error
for target $T_t$ satisfies $\epsilon_T \approx \epsilon \cdot p_v \cdot |\mathcal{X}|^{-1} \cdot |\mathcal{X}| = \epsilon \cdot p_v$;
more precisely, for the two regimes below:

\smallskip
\noindent\textbf{Typical event} ($p_v = 1/|\mathcal{X}| = 10^{-3}$).
Prop.~\ref{prop:binary_projection} gives $\epsilon_T \approx \epsilon/|\mathcal{X}| = 5\times 10^{-5}$.
The noise floor evaluates to
\begin{equation}
    \tau_{\epsilon_T}
    = 2\Phi\!\left(\sqrt{\epsilon_T/2}\right)
 = 0.053 \; \text{nats},
\end{equation}
Strong faithfulness (Def.~\ref{def:strong_faithfulness}) then requires
any true causal association to satisfy 
$I(T_t; E_{t-\ell}\mid X_{<t}) > 2\tau_{\epsilon_T} \approx 0.01\ \text{nats}$.
\smallskip
Two consequences follow directly.
First, the noise floor \emph{decreases} as $p_v$ decreases: rarer events 
admit a tighter decision boundary, because the binary projection of a 
low-probability target retains only a small slice of the total 
approximation error $\epsilon$.
Second, both floors vanish as $\epsilon \to 0$, confirming that an 
exact oracle recovers every causal edge regardless of event frequency.
This counterintuitive scaling — larger vocabularies yield lower per-test 
noise floors.
\end{example}

\subsection{Limitations}\label{appendix:limitation}
We now include a discussion regarding the main assumptions taken in this paper and the one that we exclude.

\subsubsection{Causal Sufficiency}\label{sec:limitation_causal_sufficency}
A fundamental assumption in causal discovery is \textit{Causal Sufficiency} (Assumption~\ref{ass:sufficiency})—the premise that no unobserved confounders influence the system. Since \textsc{seq2cause} relies on pre-trained backbones which may have learned from noisy or incomplete data, we empirically evaluate its robustness under controlled violations of causal sufficiency, focusing on two realistic forms of hidden confounding.

\paragraph{Measurement Error (Noise Injection).} In Fig. \ref{fig:limitations}(a), we simulate measurement error by randomly replacing valid tokens in the history with noise ($P_{noise}$). While Recall naturally degrades as the true causal parents are obscured, \textbf{Precision remains high} ($>0.8$) even when 40\% of the context is corrupted. This confirms that \textsc{seq2cause} does not "rationalize" noise; if the causal signal $X \to Y$ is destroyed by measurement error, the model assigns $CMI \approx 0$ rather than hallucinating a spurious link.

\paragraph{Missing Intermediaries (Temporal Drops).} In the same Fig. \ref{fig:limitations}(b), we simulate missing data by randomly dropping time steps, effectively hiding intermediate nodes in the causal chain ($X \to Z_{hidden} \to Y$). This is a more critical scenario where the conditioning sets are broken. We observe that \textsc{seq2cause} is robust to moderate data loss ($P_{drop} < 0.2$). 
\begin{figure}[!h]
    \centering
    \includegraphics[width=0.95\linewidth]{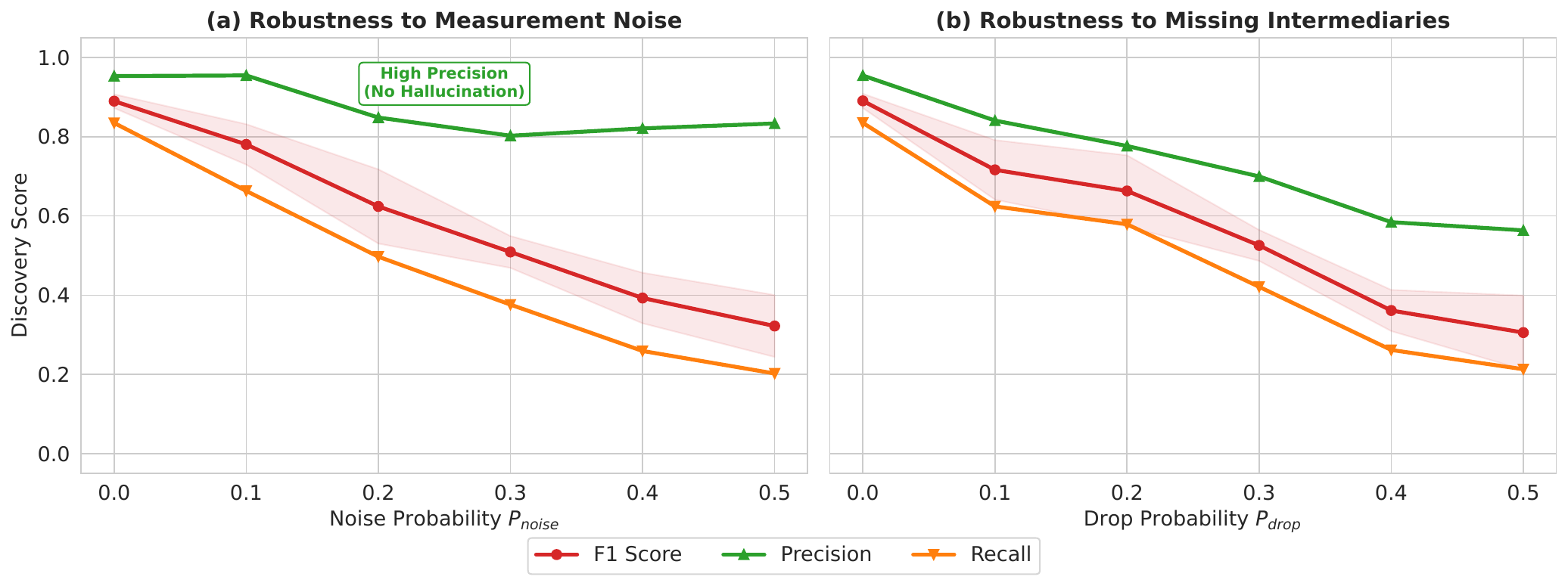}
\caption{\textbf{Robustness to Hidden Confounding.} Evaluation of \textsc{seq2cause} under violations of causal sufficiency. \textbf{(a)} Measurement Error: Random noise is injected into the context. Precision stays high, indicating resistance to hallucination. \textbf{(b)} Temporal Drops: Time steps are randomly deleted, thus conditioning sets are broken. \textsc{seq2cause} still recovers the correct sample-summary causal graph structure despite missing intermediaries but the discovery scores quickly decrease.}
\label{fig:limitations}
\end{figure}
\subsubsection{Temporal Precedence \& Instantaneous Effects}
We assumes that causal influence respects temporal precedence, which means observation are perfectly recorded over time and therefore does not model instantaneous causal effects between events occurring at the same time index. This assumption is standard in sequential causal discovery and ensures that the recovered causal graph is acyclic and identifiable in the single observed sequence setting.

From a theoretical standpoint, instantaneous effects are not identifiable from a single observed trajectory without additional parametric assumptions, repeated samples, or access to interventions. In practice, apparent simultaneity often \emph{arises from time discretization}, logging resolution, or batching effects. \textsc{seq2cause} interprets such cases \emph{through the earliest observable temporal ordering}, yielding a conservative but identifiable causal structure. As a result, the recovered summary graph captures directed causal influence with positive temporal delay, rather than true simultaneity.

\subsubsection{Consistency Through Time}
We assume \emph{consistency through time}: the causal mechanisms governing the generation of events are invariant across time indices. In other words, causal directionality does not reverse over time. For instance, if \(A \rightarrow B\), then latter it is assume that \(B \not\rightarrow A\)

This assumption is strictly weaker than stationarity~\cite{assaad_survey_ijcai_cd_time_series}. While the marginal distribution of $\{X_t\}$ may vary over time, the underlying causal dependencies—encoded by the directed edges of the instance time causal graph remain stable. This form of causal invariance is standard in sequential settings and underlies the validity of summary causal graphs that collapse time-indexed relations into event-to-event dependencies.

\subsubsection{No Inter-Label Effects}
We assume $Y_j \perp Y_k \mid X_{<L}$ for all $j \neq k$:
labels do not cause one another, they are jointly caused by 
the event history.
This is not a restrictive assumption in our setting — it is 
structurally implied by the two-chain DGP (Eq.~2).
Since outcomes are realized at the end of the sequence and 
the event chain is autonomous (no $Y \to X$ edges exist), 
any apparent statistical dependence between $Y_j$ and $Y_k$ 
is fully mediated by their shared event causes in $X_{<L}$.
Conditioning on $X_{<L}$ therefore blocks every path between 
labels, making inter-label CI-tests redundant rather than 
informative.
This assumption is standard in multi-label causal discovery 
where the label space is defined exogenously~\citep{Wu_Jiang_Yu_Chen_Miao_2020, 
learningcommoncausalvarlabel}, and matches the real-world 
structure of our target domains: a vehicle failure $Y_j$ 
and a separate failure $Y_k$ are both downstream of the 
same diagnostic event stream, not causally related to each 
other.

Nevertheless, we can consider for the \(\mathcal{X} \to \mathcal{Y}\) a second light phase to check inter-label dependencies. Let's consider the Markov Boundaries \(MB_1 =  [X_1, X_3], MB_2 = [X1, X_2]\) of two labels \(Y_1, Y_2\). We propose to investigate a 'Phase 2', focusing on inter-label dependencies through simulated interventions. For instance, if we consider a sequence \(S_1\) of two labels \(Y_1, Y_2\) with the MB above, we can perform counterfactual interventions by applying \(do(X_1=0), do(X_3=0)\) to \(S_1\) (under A\ref{ass:sufficiency}). Then we would observe the average change in the likelihood of \(Y_1\)  which if it is non-zero, would indicate a dependence between \(Y_1\) and \(Y_2\).

\subsubsection{Not-required for \textsc{seq2cause}}
We now list the notable assumption that we \textbf{don't take} in this paper. 

\paragraph{Stationarity}
The stationarity assumption states that the generative process \( \{X_t\}\) does not change with respect to time. Stationarity of the underlying process is a common simplifying assumption in time series causal discovery~\cite{assaad_survey_ijcai_cd_time_series}, but causal structure itself—defined in terms of temporal precedence and directed edges—is a property of the generative mechanisms and does not by itself imply stationarity of the observed sequence.

Importantly, \textsc{seq2cause} does not require stationarity of the observed sequence. In practice, if using autoregressive Transformers as the AR Model, they can represent non-stationary distributions through contextualized representations (e.g., positional or time embeddings~\cite{transformerhawkeprocess}), allowing the model to adapt its predictions to evolving regimes without assuming time-invariant marginals. This enables modeling evolving regimes common in real-world logs or patients trajectories.

\paragraph{Parametric Assumption}
No parametric form is assumed for the transition dynamics beyond the expressivity of the autoregressive model

\subsubsection{Broader Impact Limitations}\label{app:broader_impact_limitation}
\paragraph{Risks under degraded models and imbalanced data.}
\textsc{seq2cause} inherits the biases of its autoregressive backbone: if the AR model is trained on imbalanced event data---where certain event types are overrepresented due to systematic logging biases, sensor placement, or population skew---the estimated conditionals $P_\theta(X_t \mid X_{<t})$ will be poorly calibrated, inflating $\epsilon$ non-uniformly across the vocabulary. By Thm.~\ref{thm:cmi_error_bound}, this inflates the CMI error bound precisely for the under-represented pairs, causing the framework to systematically miss causal edges involving minority event types while confidently recovering edges among frequent ones. In safety-critical deployments (e.g., vehicle diagnostics, clinical pathways), this failure mode is insidious: the returned causal graph appears high-precision but silently omits rare-but-critical failure chains. Practitioners should monitor the excess cross-entropy loss and treat any causal graph produced with $\epsilon > 0.1$ as unreliable (Fig.~\ref{fig:robustness_and_scale}(a). More broadly, deploying \textsc{seq2cause} on biased event logs without human oversight risks reinforcing existing system biases---for instance, concluding that a fault code is non-causal simply because it was underlogged in the training fleet.

\section{Proofs}\label{app:proofs}

\subsection{Proof of Proposition~\ref{prop:consistency_cmi} (CMI Estimator Consistency)}

\begin{proposition}[Convergence to the $\epsilon$-Proxy]\label{prop:consistency_cmi}
The estimator $\hat{I}_N$ is a consistent estimator of the \(\epsilon\)-oracle induced CMI denoted as $I_\theta$. By the Strong Law of Large Numbers, as $N \to \infty$:
$$\hat{I}_N \xrightarrow[N \to +\infty]{\text{a.s.}} I_\theta(E_{t+1}; E_t \mid X_{<t})$$
\end{proposition}

\begin{proof}\label{proof:consistency_estimator_cmi}
The particles $x^{(l)}_{<t}$ are sampled directly from the model $f_\theta$.
Let $f_\theta(x^{(l)}_{<t})$ represents the estimation CMI for a fixed history \(x_{<t}\) and \(I_\theta\) the CMI with the approximated distribution \(P_\theta\). Expressing this as a difference of conditional entropies: 
\begin{equation}\label{eq:info_gain_bounded}
\begin{aligned}
0 
&\leq \mathbb{E}_{e_t \sim P_\theta} I_G\!\left(E_{t+1}, e_t| x^{(l)}_{<t}\right) \\
&= H_\theta\!\left(E_{t+1} | x^{(l)}_{<t}\right)
   - H_\theta\!\left(E_{t+1}| E_t, x^{(l)}_{<t}\right) \\
&\leq H_\theta\!\left(E_{t+1}\right)
\leq \log{2}.
\end{aligned}
\end{equation}
Thus the posterior variance of \(f_\theta(x^{(l)}_{<t})\) satisfies \(\sigma^2_{f} \triangleq \mathbb{E}_{x_{<t}}[f_\theta^2(x_{<t})] - I_\theta^2(f_\theta) < +\infty\) \cite{Doucet2001} then the variance of \(\hat{I}_N(f)\) is equal \(\frac{\sigma^2_{f}}{N}\) and from the strong law of large numbers: 
\begin{align}
\hat{I}_N &\xrightarrow[N \to +\infty]{\text{a.s.}} 
\mathbb{E}_{e_t \sim P_\theta, x_{<t} \sim P_\theta}\!\left[ I_G(E_{t+1}, e_t \mid x_{<t}) \right] 
\triangleq I_\theta(f_\theta)
\end{align}
\end{proof}

arx\subsection{Proof of Proposition \ref{prop:binary_projection} (Binary Projection of \(\epsilon\))}
\begin{proposition}[Binary Projection Inequality]\label{prop:binary_projection}
Let $P, Q$ be distributions over $\mathcal{X}$ with $D_{\mathrm{KL}}(P \| Q) = \epsilon$. For any event type $v \in \mathcal{X}$, let $p_v = P(X{=}v)$, $q_v = Q(X{=}v)$, and $T = \mathbf{1}[X{=}v]$. Then the model approximation error \(\epsilon_T\) for the binary target \(T\) is:
\begin{equation}\label{eq:exact_decomp}
    \epsilon_T \;=\; D_{\mathrm{KL}}\!\big(\mathrm{Bern}(p_v) \;\big\|\; \mathrm{Bern}(q_v)\big) \;=\; \epsilon \;-\; (1 - p_v) \cdot D_{\mathrm{KL}}\!\big(P_{\neg v} \;\big\|\; Q_{\neg v}\big),
    \end{equation}
    where $P_{\neg v}(u) := p_u\,/\,(1 - p_v)$ for $u \neq v$ is the conditional distribution over $\mathcal{X} \setminus \{v\}$, and similarly for $Q_{\neg v}$.
\end{proposition}

\begin{proof}
We apply the chain rule for relative entropy (\cite{cover1999elements}, Thm.~2.5.3). Any distribution over $\mathcal{X}$ can be viewed as a two-stage draw: first sample the binary indicator $B = \mathbf{1}[X{=}v] \sim \mathrm{Bern}(p_v)$, then, if $B = 0$, draw from the conditional $P_{\neg v}$ over the remaining $|\mathcal{X}| - 1$ types; if $B = 1$, output $v$ deterministically. The chain rule gives:
\begin{align}\label{eq:chain_rule_app}
D_{\mathrm{KL}}(P \| Q) = D_{\mathrm{KL}}\!\big(\mathrm{Bern}(p_v) \| \mathrm{Bern}(q_v)\big) + p_v \cdot D_{\mathrm{KL}}\!\big(P(\cdot \mid B{=}1) \;\|\; Q(\cdot \mid B{=}1)\big) + \\ (1 - p_v) \cdot D_{\mathrm{KL}}\!\big(P_{\neg v} \;\|\; Q_{\neg v}\big).
\end{align}

Conditioned on $B = 1$, both distributions place all mass on $v$, so $D_{\mathrm{KL}}(P(\cdot \mid B{=}1) \| Q(\cdot \mid B{=}1)) = D_{\mathrm{KL}}(\delta_v \| \delta_v) = 0$. Substituting $D_{\mathrm{KL}}(P \| Q) = \epsilon$ and rearranging yields Eq.~(\ref{eq:exact_decomp}).
 
Since $D_{\mathrm{KL}}(P_{\neg v} \| Q_{\neg v}) \geq 0$, the immediate consequence is $\epsilon_T \leq \epsilon$. The gap $\epsilon - \epsilon_T = (1 - p_v) \cdot D_{\mathrm{KL}}(P_{\neg v} \| Q_{\neg v})$ quantifies the fraction of the total categorical error that is \emph{internal} to $\mathcal{X} \setminus \{v\}$ and irrelevant to the binary CI test for target $v$.
 
\end{proof}
 
\subsection{Proof of Thm.~\ref{thm:cmi_error_bound} (CMI Bounded by Model Approximation Error)}

\begin{proof}
By definition, the Conditional Mutual Information is the difference of two conditional entropies~\cite{cover1999elements}:

    $$I(T_t; E_{t-\ell} | X_{<t}) = H(T_t | X_{<t}) - H(T_t | E_{t-\ell}, X_{<t})$$
    With \(T_t \in \{E_t\} \cup \{Y_1, \dots, Y_c\}\) is the target (effect) and \(E_{t-\ell}\) is the cause.
    Our framework operates in the Teacher Forcing regime. We do not sample the history $X_{<t}$ from the model's joint distribution. Instead, we estimate the CMI conditioned on the observed history $x_{<t}^{(l)}$. Consequently, the relevant error metric is the per-step conditional divergence at time $t$, given the fixed history. 

    Let $\Delta = |I - I_\theta|$ be the CMI estimation error. By the triangle inequality:

    $$\Delta \leq \underbrace{|H_P(T_t \mid X_{<t}) - H_\theta(T_t \mid X_{<t})|}_{\text{Term A}} + \underbrace{|H_P(T_t \mid E_{t-\ell}, X_{<t}) - H_\theta(T_t \mid E_{t-\ell}, X_{<t})|}_{\text{Term B}}$$

    We apply the sharp continuity bound for conditional entropy in classical systems (\cite{Winter_2016}, Lem. 2) with a small total variation distance \(\text{TV}(P, P_\theta))  \leq 1/2\). Let $\rho$ and $\sigma$ be the true and model distributions respectively. Let $\delta = |(P(E_t|\cdot) - P_\theta(E_t|\cdot)|_1$ be the total variation distance. Since the target variable $E_t$ is binary (\(d_A = 2\)), the bound is: $$|H_\rho - H_\sigma| \leq \delta \ln(d_A) + (1+\delta)h_b\left(\frac{\delta}{1+\delta}\right)$$

    Substituting $d_A=2$ (so $\ln 2 $ in nats since we are using cross entropy loss in PyTorch~\cite{pytorch}):$$|H_P(T_t \mid X_{<t}) - H_\theta(T_t \mid X_{<t})| \leq \delta \ln(2) + (1+\delta)h_b\left(\frac{\delta}{1+\delta}\right)$$

    Term B represents the same entropy difference conditioned on an augmented set $\{E_{t-\ell}, X_{<t}\}$. Since the target dimension $d_A$ remains 2, the same bound applies. Summing the terms:
    $$ |I - I_\theta| \leq 2\delta \ln(2) + 2(1+\delta)h_b\left(\frac{\delta}{1+\delta}\right)$$


    
    Substituting $\delta = \sqrt{\epsilon/2}$ from Pinsker's inequality, and $h_b(x)$ monotonically increasing for small $x$, we obtain the final bound in terms of the oracle score $\epsilon$ as:
    
\begin{equation}
| I - I_\theta  | \leq 2 \sqrt{\epsilon/2} \ln(2) + 2(1+\sqrt{\epsilon/2} )h_b\left(\frac{\sqrt{\epsilon/2} }{1+\sqrt{\epsilon/2} }\right)
\end{equation}    

Finally, we decompose the total error into estimation variance and approximation bias using the triangle inequality:
\begin{equation}
    | \hat{I}_N - I | = | (\hat{I}_N - I_\theta) + (I_\theta - I) | \leq \underbrace{| \hat{I}_N - I_\theta |}_{\text{Estimation Error}} + \underbrace{| I_\theta - I |}_{\text{Approximation Bias}}
\end{equation}
Given that $\hat{I}_N$ is a consistent estimator of the model's internal CMI, $I_\theta$ (Prop.~\ref{prop:consistency_cmi}). By the Strong Law of Large Numbers, $\hat{I}_N \xrightarrow{a.s.} I_\theta$ as $N \to \infty$ the stochastic estimation error vanishes, leaving only the irreducible approximation bias:
$$ \limsup_{N \to \infty} | \hat{I}_N - I | \leq 0 + |I - I_\theta|$$

We thus obtain the final bound in terms of the oracle score as:

$$ \limsup_{N \to \infty} | \hat{I}_N - I | \leq 2 \sqrt{\epsilon/2} \ln(2) + 2(1+\sqrt{\epsilon/2} )h_b\left(\frac{\sqrt{\epsilon/2} }{1+\sqrt{\epsilon/2} }\right)$$

 Using the model approximation error \(\epsilon_T\) for the binary target \(T\) which is lower than the global model approximation error \(\epsilon\), using Prop. \ref{prop:binary_projection} we have: 

\begin{equation}
\limsup_{N \to \infty} \big|\hat{I}_N - I\big| \;\leq\; 2\,\Phi\!\Big(\!\sqrt{\epsilon_T\,/\,2}\Big) \;\leq\; 2\,\Phi\!\Big(\!\sqrt{\epsilon/2}\Big),
\end{equation} 
where $\Phi(\delta) := \delta \ln 2 + (1{+}\delta)\,h_b\!\big(\tfrac{\delta}{1+\delta}\big)$, $h_b$ is the binary entropy.

    
This confirms that minimizing the cross-entropy loss ($\epsilon$) directly minimizes the upper bound on structural causal error using the CMI as causal strength.
\end{proof}

\subsection{Proof of Thm.~\ref{thm:soudness_sample_level} (Soundness of the Sample-Level)}
\begin{proof}
We proceed by induction on the time index $t \in \{1, \dots, L\}$, 
exploiting temporal precedence~(A\ref{ass:temporal}) 
to reduce the parent search at each step to the strict past 
$\{E_{t-\ell}\}_{\ell \geq 1}$.
The goal is to show that the CMI-based inclusion rule recovers 
$\widehat{\mathrm{Pa}}(T_t) = \mathrm{Pa}_{\mathcal{G}^s}(T_t)$ 
for every target $T_t \in \mathcal{T}_t 
= \{E_{t+1}\} \cup \{Y_1, \dots, Y_c\}$, 
and hence the complete joint graph 
$\mathcal{G}^s = \mathcal{G}^s_{\mathcal{XX}} \cup \mathcal{G}^s_{\mathcal{XY}}$ (Fig.~\ref{fig:sample_level_graph}).

\paragraph{Base case ($t = 1$).}
Consider any target $T_1 \in \mathcal{T}_1$.
By temporal precedence~(A\ref{ass:temporal}), 
no event precedes $t=1$; the candidate set 
$\{E_{1-\ell}\}_{\ell \geq 1}$ is empty.
\textsc{seq2cause} therefore returns 
$\widehat{\mathrm{Pa}}(T_1) = \emptyset = \mathrm{Pa}_{\mathcal{G}^s}(T_1)$.

\paragraph{Inductive step.}
Fix $t \geq 2$ and assume that for all $j < t$ the parent sets have been correctly identified. Let $T_t \in \mathcal{T}_t$ be an arbitrary target and let 
$E_{t-\ell}$, $\ell \in \{1,\dots,t-1\}$, be a candidate cause.

\emph{For $T_t = E_{t+1}$ with $\ell > 1$}, the CMI is evaluated after do-intervening on the intermediate events (Def.~\ref{def:do_in_sequences_expected}); this blocks the front-door \(X_{t-\ell+1:t-1}\) by randomizing the mediators. Using the lagged information gain \(I^L_G\) (Eq.~\ref{eq:lagged_info_gain}), the CMI is replaced and we condition on \(X_{<t-\ell}\) while controlling for the mediators. The bound of Eq.~\ref{eq:cmi_bound} applies.

\emph{For $T_t = Y_j$ or $\ell = 1$}, the standard CMI of 
Eq.~\eqref{eq:ci_test_unified} applies without modification, no intervention needed.
In both cases the inclusion rule is identical and the two structural cases below apply.

\textbf{$E_{t-\ell}$ is a true parent 
($E_{t-\ell} \in \mathrm{Pa}_{\mathcal{G}^s}(T_t)$).}
By $\epsilon$-strong faithfulness 
(Def.~\ref{def:strong_faithfulness}), every true causal 
association satisfies
\begin{equation}
    I\!\left(T_t;\, E_{t-\ell} \mid X_{<t}\right)
    > 4\Phi\!\left(\sqrt{\epsilon_T/2}\right).
\end{equation}
By Thm.~\ref{thm:cmi_error_bound}, the Monte Carlo estimator 
satisfies $\limsup_{N \to \infty} \left|\hat{I}_N - I\right|
\leq 2\Phi\!\left(\sqrt{\epsilon_T/2}\right)$,
so asymptotically $\hat{I}_N > 2\Phi\!\left(\sqrt{\epsilon_T/2}\right)$
and the edge is \emph{accepted}.

\textbf{$E_{t-\ell}$ is not a parent 
($E_{t-\ell} \notin \mathrm{Pa}_{\mathcal{G}^s}(T_t)$).}
By the two-chain factorization~(Eq.~\ref{eq:joint_dgp}) 
and causal sufficiency~(A\ref{ass:sufficiency}), 
conditioning on the full history $X_{<t}$ blocks every active path 
from $E_{t-\ell}$ to $T_t$ via the Causal Markov Condition.
Hence the true CMI vanishes:
\begin{equation}
    I\!\left(T_t;\, E_{t-\ell} \mid X_{<t}\right) = 0.
\end{equation}
By Remark~\ref{remark:full_history_conditioning}, 
full-history conditioning also closes any backdoor path that would 
remain open under the coarsened binary history $E_{<t}$.
Thm.~\ref{thm:cmi_error_bound} then gives
$\hat{I}_N \leq 2\Phi\!\left(\sqrt{\epsilon_T/2}\right)$
asymptotically, so the edge is \emph{rejected}.

Since the accept/reject decision is correct for every candidate 
$E_{t-\ell}$, $\ell = 1, \dots, t-1$, independently of whether 
$T_t = E_{t+1}$ (the $\mathcal{X} \to \mathcal{X}$ regime) 
or $T_t = Y_j$ (the $\mathcal{X} \to \mathcal{Y}$ regime),
we obtain $\widehat{\mathrm{Pa}}(T_t) = \mathrm{Pa}_{\mathcal{G}^s}(T_t)$.

By induction, $\widehat{\mathrm{Pa}}(T_t) = \mathrm{Pa}_{\mathcal{G}^s}(T_t)$ 
holds for all $t = 1, \dots, L$ and all $T_t \in \mathcal{T}_t$.
The sample-time causal graph $\mathcal{G}^s_{\mathcal{XX}}$ is the union of 
the intra-chain parent sets; the sample Markov boundary graph 
$\mathcal{G}^s_{\mathcal{XY}}$ is the union of the cross-chain parent sets.
Their union $\mathcal{G}^s = \mathcal{G}^s_{\mathcal{XX}} \cup \mathcal{G}^s_{\mathcal{XY}}$ 
is therefore recovered. 
\textit{Note:} As 
$\epsilon \to 0$ since $\Phi\!\left(\sqrt{\epsilon_T/2}\right) \to 0$, under standard faithfulness, we recover the joint graph for perfect density estimation.
\end{proof}

\subsection{Proof of Lem.~\ref{lem:mdl_deviation} 
(Data Encoding Cost Deviation under the \(\epsilon\)-Oracle)}
\label{proof:threshold_mdl}

\begin{proof}
Fix a candidate edge \(u \to v\) and a matched position 
\((k,t) \in \mathcal{M}_{uv}\). Let \(Z_+^{(k,t)}\) and \(Z_-^{(k,t)}\) denote 
the conditioning sets with and without the candidate cause 
\(E^{(k)}_{t-\ell} = u\) (randomized mediators for \(\mathcal{X}\to\mathcal{X}\); 
trivial for \(\mathcal{X}\to\mathcal{Y}\) by the two-chain DGP). To lighten 
notation, write
\[
p_\pm \;=\; P(T^{(k)}_t = t_{t,v}^{(k)} \mid Z_\pm^{(k,t)}), 
\qquad 
q_\pm \;=\; P_\theta(T^{(k)}_t = t_{t,v}^{(k)} \mid Z_\pm^{(k,t)}),
\]
so \(\ell_t^{(k)} = \log p_+ - \log p_-\) and \(\ell_t^{(k),\theta} = \log q_+ - \log q_-\).

\medskip
By the triangle inequality,
\begin{equation}\label{eq:decomp_lemma1}
\bigl|\ell_t^{(k),\theta} - \ell_t^{(k)}\bigr|
\;\leq\; |\log q_+ - \log p_+| \;+\; |\log q_- - \log p_-|.
\end{equation}
Each summand is handled identically; we bound the generic term 
\(|\log q - \log p|\) for \(p, q \in (0,1)\).

\medskip
For any \(a, b \in (0, 1)\), 
\(|\log a - \log b| \leq |a-b|/\min(a, b)\) 
\cite{cover1999elements}. Hence
\begin{equation}\label{eq:logratio_step}
|\log q - \log p| \;\leq\; \frac{|q - p|}{\min(p, q)}.
\end{equation}

\medskip
The binary target \(T_v\) is a deterministic function of the full categorical 
outcome \(X_t\), so by the data processing inequality for KL divergence,
\[
D_{\mathrm{KL}}\bigl(\mathrm{Bern}(p) \,\|\, \mathrm{Bern}(q)\bigr) 
\;\leq\; D_{\mathrm{KL}}\bigl(P(\cdot \mid Z) \,\|\, P_\theta(\cdot \mid Z)\bigr) 
\;\leq\; \epsilon_T^{(k,t)} \;\leq\; \epsilon,
\]
where \(\epsilon_T^{(k,t)}\) is the local binary approximation error 
(Prop.~\ref{prop:binary_projection}). Pinsker's inequality then gives
\begin{equation}\label{eq:pinsker_step}
|q - p| \;\leq\; \mathrm{TV}\bigl(\mathrm{Bern}(p), \mathrm{Bern}(q)\bigr) 
\;\leq\; \sqrt{\epsilon_T^{(k,t)} / 2}.
\end{equation}

\medskip
The denominator \(\min(p, q)\) involves \(p\), which we cannot access. However, 
\(q\) is computed directly from the frozen AR model, and combining 
\eqref{eq:pinsker_step} with the triangle inequality yields
\begin{equation}\label{eq:p_lower}
p \;\geq\; q - |q - p| \;\geq\; q - \sqrt{\epsilon_T^{(k,t)}/2}.
\end{equation}
Two cases arise.
\emph{Case A (\(q > \sqrt{\epsilon_T^{(k,t)}/2}\))}: 
\eqref{eq:p_lower} certifies \(p > 0\), and
\begin{equation}\label{eq:min_bound_caseA}
\min(p, q) \;\geq\; q - \sqrt{\epsilon_T^{(k,t)}/2} \;>\; 0.
\end{equation}

\emph{Case B (\(q \leq \sqrt{\epsilon_T^{(k,t)}/2}\))}: 
The Pinsker bound cannot certify \(p > 0\); the log-ratio may be unbounded. 
We declare this position under-resolved at the current oracle quality 
and exclude it from the MDL accumulation. At oracle-optimal quality 
(\(\epsilon \to 0\)) Case B never triggers: any position with 
\(q = P_\theta(T_v \mid Z) > 0\) is eventually resolved as \(\epsilon_T^{(k,t)}\) 
shrinks below \(2 q^2\).

\medskip
Combining \eqref{eq:logratio_step}--\eqref{eq:min_bound_caseA} over both 
conditioning sets and both non-excluded cases,
\begin{equation}\label{eq:per_pos_bound}
\bigl|\ell_t^{(k),\theta} - \ell_t^{(k)}\bigr|
\;\leq\; \sqrt{\epsilon_T^{(k,t)}/2} \cdot 
\underbrace{\left[\frac{\mathbf{1}\{q_+ > \sqrt{\epsilon_T^{(k,t)}/2}\}}
{q_+ - \sqrt{\epsilon_T^{(k,t)}/2}} 
+ \frac{\mathbf{1}\{q_- > \sqrt{\epsilon_T^{(k,t)}/2}\}}
{q_- - \sqrt{\epsilon_T^{(k,t)}/2}}\right]}_{\displaystyle \kappa_t^{(k)}(u,v)}.
\end{equation}
Both factors of \(\kappa_t^{(k)}(u,v)\) are computable from the two forward passes 
through \(P_\theta\) already required to evaluate \(\ell_t^{(k),\theta}\); no 
additional model queries are needed.

\medskip
Summing \eqref{eq:per_pos_bound} over matched positions and using 
\(\epsilon_T^{(k,t)} \leq \epsilon_T\),
\begin{equation}\label{eq:compute_threshold}
\bigl|\Delta L^\theta - \Delta L\bigr| 
\;\leq\; \sqrt{\epsilon_T/2} \cdot 
\underbrace{\sum_{(k,t)\in \mathcal{M}_{uv}} \kappa_t^{(k)}(u,v)}_{\displaystyle \eta_{uv}}.
\end{equation}
This yields the claim of Lem.~\ref{lem:mdl_deviation}.
\end{proof}

\begin{remark}[\(\eta_{uv}\) is a well-defined function of the model and data]
Every term in \(\eta_{uv}\) is observable without access to the true distribution 
\(P\): \(q_\pm\) is computed from the frozen AR forward pass, and \(\epsilon_T\) 
is read off the training diagnostics (Appendix~\ref{app:epsilon_monitor}). The 
indicator on each term makes \(\eta_{uv}\) robust to under-resolved positions — 
excluding them rather than letting a spurious \(1/0\) dominate the sum.
\end{remark}

\section{Related Work}\label{appendix:related_work}
\subsection{Comparison}
\begin{table*}[!h]
\centering
\caption{\textbf{Comparison of causal discovery methods for event 
sequences}. \textbf{Discrete}: Operates on symbolic sequences. 
\textbf{Non-Param.}: No assumed functional form. 
\textbf{Single Stream}: Inference from one observed sequence. 
$\mathcal{X}{\to}\mathcal{X}$: Event-to-event dependencies. 
$\mathcal{X}{\to}\mathcal{Y}$: Event-to-outcome. 
\textbf{Scalable}: Tractable at $|\mathcal{X}|{>}10^3$ 
(linear complexity in $|\mathcal{X}|$). 
\textbf{RCA}: Sample-level root-cause identification from 
a single sequence without topology assumptions or interventional data.}
\small
\setlength{\tabcolsep}{4pt}
\begin{tabular}{lcccccccc}
\toprule
\textbf{Method} 
& \parbox{0.8cm}{\centering \textbf{Discrete}}
& \parbox{1.0cm}{\centering \textbf{Non-}\textbf{Param.}} 
& \parbox{1.1cm}{\centering \textbf{Single} \textbf{Stream}} 
& \parbox{1.0cm}{\centering $\mathcal{X}{\to}\mathcal{X}$} 
& \parbox{1.0cm}{\centering $\mathcal{X}{\to}\mathcal{Y}$}
& \parbox{1.1cm}{\centering \textbf{Scalable}}
& \parbox{0.8cm}{\centering \textbf{RCA}} \\
\midrule
Constraint-based (PCMCI, FCI) 
& $\times$ & \cmark & $\times$ & \cmark & $\times$ & $\times$ & $\times$ \\
Score-based (DYNOTEARS) 
& $\times$ & $\times$ & $\times$ & \cmark & $\times$ & $\times$ & $\times$ \\
Granger (TCDF, CAUSE) 
& $\times$ & $\times$ & $\times$ & \cmark & $\times$ & $\times$ & $\times$ \\
Noise-based (VarLiNGAM) 
& $\times$ & $\times$ & $\times$ & \cmark & $\times$ & $\times$ & $\times$ \\
Hawkes / TPP (THP, SHTP)
& \cmark & $\times$ & $\times$ & \cmark & $\times$ & $\times$ & $\times$ \\
Info-Theoretic (NPHC, CASCADE) 
& \cmark & \cmark & $\times$ & \cmark & $\times$ & $\times$ & $\times$ \\
Multi-label (IAMB, CMB)
& \cmark & $\times$ & $\times$ & $\times$ & \cmark & $\times$ & $\times$ \\
RCA methods (AERCA~\cite{han2025root}, CoE~\cite{chain_of_event_2024_sre})
& $\times$ & $\times$ & $\circ$ & $\circ$ & $\times$ & $\times$ & $\circ^\dagger$ \\
\midrule
\textbf{\textsc{seq2cause} (Ours)} 
& \cmark & \cmark & \cmark & \cmark & \cmark & \cmark & \cmark \\
\bottomrule
\end{tabular}
\footnotesize{$\dagger$~AERCA performs RCA on continuous multivariate streams 
with known dimensionality $|\mathcal{X}| \leq 100$; 
CoE requires a known service topology and labeled incident history.}
\label{tab:related_work}
\end{table*}

\subsection{Generative Model for CI-testing}

Our work can be viewed as the discrete sequential analog of the generative-CI-test lineage initiated by  \
\cite{NIPS2017_02f03905_cit_class} and extended via \cite{gcit_gan_neurips2019, ci_flows_2023, yang2025cdcit} and score-based variants. These methods learn a fresh conditional generator per CI query over continuous i.i.d. data; \textsc{seq2cause} instead repurposes a single, pretrained, autoregressive model whose conditional is native to the data modality (discrete sequences), amortizing density estimation across all candidate edges. We provide theoretical bound analysis (Thm.~\ref{thm:cmi_error_bound}) for the sample-level and extend it to the population-level with the MDL error-correction term (Lem.~\ref{lem:mdl_deviation}) to handle population aggregation.

\subsection{LLM as Causal Reasoner}

A parallel line of work treats LLMs themselves as causal oracles, either answering pairwise causal-direction queries from variable metadata \cite{vashishtha2025causal}, serving as imperfect experts~\cite{takayama2025integrating}. These methods rely on the semantic knowledge encoded in LLM pretraining and are critically limited by hallucination, prompt sensitivity, and the absence of statistical guarantees — as recently formalized by \cite{mirzadeh2025gsmsymbolic}. \textsc{seq2cause} uses the AR model in a fundamentally different way: not as a textual oracle queried in natural language, but as a calibrated conditional density estimator over the event alphabet, with CMI tests grounded in the model's explicit output distribution and requires no assumptions regarding the data generating process~\cite{tf_causalinterpretation_neurips_2023}. The guarantee travels through \(\epsilon\), not through the prompt. Therefore, \textsc{seq2cause} aim to enhance an LLM's reasoning capability via a causal graph, similarly to \cite{math2026neurosymbolic}.

\subsection{Data Settings}
\begin{table}[!h]
    \centering
    \caption{\textbf{Data Paradigm Comparison.} Contrast between traditional event sequence causal discovery (multivariate) and our single sequence setting (Session-based/NLP-like). Standard methods require the structure on the left and scale poorly to the structure on the right.}
    \label{tab:data_comparison}
    \resizebox{\columnwidth}{!}{%
    \begin{tabular}{@{}lll@{}}
        \toprule
        \textbf{Feature} & \textbf{Traditional Approach} & \textbf{Ours (seq2cause)} \\
        & (e.g., Hawkes, Granger, PCMCI) & (Autoregressive Model) \\
        \midrule
        \textbf{Input Format} & \textbf{Long / Vertical Stream} & \textbf{Wide / Horizontal Batches} \\
        \textbf{Structure} & 
        \begin{minipage}{0.35\columnwidth}
            \scriptsize
            \texttt{Time \hspace{2pt} Event\_Type} \\
            \texttt{0.0 \hspace{5pt} E\_A} \\
            \texttt{0.4 \hspace{5pt} E\_B} \\
            \texttt{1.2 \hspace{5pt} E\_C} \\
            \texttt{... \hspace{8pt} ...}
        \end{minipage} 
        & 
        \begin{minipage}{0.45\columnwidth}
            \scriptsize
            \texttt{Seq\_ID \hspace{2pt} Sequence (Tokens) \hspace{4pt} Time} \\
            \texttt{0 \hspace{15pt} [E\_A, E\_B, E\_A, ...] [0.0, 1.3, 1.4, ...]} \\
            \texttt{1 \hspace{15pt} [E\_C, E\_D, E\_G, ...] [0.0, 1.2, 5.9, ...]}\\
            \texttt{... \hspace{12pt} ...}
        \end{minipage} \\
        \midrule
        \textbf{Dimensionality} & Low \(|\mathcal{X}| \leq 100\) & Massive \(|\mathcal{X}| > 1,000\) \\
        \textbf{Discovery Scope} & Global (Graph of all processes) & Local / Sample (Summary Graph of specific trace) \\
        \textbf{Vocabulary Complexity} & \text{Often} \(O(|\mathcal{X}|^2)\) or \(O(|\mathcal{X}|^3)\) & \(O(L \cdot |\mathcal{X}|)\) (Inference) \\
        \bottomrule
    \end{tabular}%
    }
\end{table}

\section{Additional Visualization}
\subsection{Causal Graph}
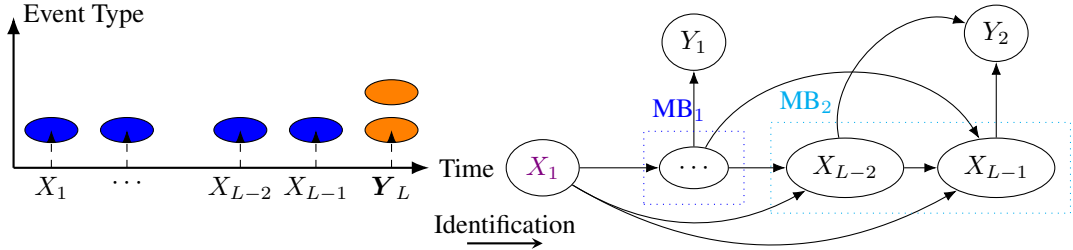
\begin{figure}[!h]
    \centering
    \begin{tikzpicture}

        \draw[thick] (-2,0) -- (3.5,0) node[right] {Time};
        \draw[thick] (-2,0) -- (-2,2) node[right] {Event Type};

        \node[state,fill=blue] (x1) at (-1.5,0.5) {};
        \node[state,fill=blue] (x2) at (-0.5,0.5) {};
        \node[state,fill=blue] (x3) at (1,0.5) {};
        \node[state,fill=blue] (x4) at (2 ,0.5) {}; 

        \node[state,fill=orange] (y) at (3,0.5) {}; 
        \node[state,fill=orange] (y) at (3, 1) {}; 

        \draw[dashed] (-1.5,0) -- (-1.5,0.5);
        \draw[dashed] (-0.5,0) -- (-0.5,0.5);
        \draw[dashed] (1,0) -- (1,0.5);
        \draw[dashed] (2,0) -- (2,0.5);
        \draw[dashed] (3,0) -- (3,0.5);

        \node[below] at (-1.5,0) {\(X_1\)};
        \node[below] at (-0.5,0) {\(\cdots\)};

        \node[below] at (1,0) {\(X_{L-2}\)};
        \node[below] at (2,0) {\(X_{L-1}\)};

        \node[below] at (3,0) {\(\boldsymbol{Y}_{L}\)};

        \draw[thick,->] (4,-1) -- (5,-1) node[midway,above] {Identification};

        \node[state] (cx1) at (5,0) {\textcolor{violet}{\(X_{1}\)}};
        \node[state] (cxd) at (7,0) {\(\cdots\)};
        \node[state] (cx2) at (9,0) {\(X_{L-2}\)};
        \node[state] (cxi) at (11,0) {\(X_{L-1}\)};
        \node[state] (cy) [above =of cxd] {\(Y_1\)};
        \node[state] (cy2) [above =of cxi] {\(Y_2\)};

        \path (cx1) edge (cxd);
        \path (cxd) edge  (cy);
        \path (cx1) edge[bend right=30] (cx2);
        \path (cx1) edge[bend right=30] (cxi);
        \path (cx2) edge (cxi);
        \path (cx2) edge[bend left=60] (cy2);
        \path (cxd) edge (cx2);
        \path (cxi) edge (cy2);
        \path (cxd) edge[bend left=60] (cxi);

        \node[draw=blue,dotted,fit=(cxd) (cxd), inner sep=0.2cm] (mb1) {};
        \node[anchor=south west, blue] at (mb1.north west) {\(\text{MB}_1\)};


        \node[draw=cyan,dotted,fit=(cx2) (cxi), inner sep=0.2cm] (mb2) {};
        \node[anchor=south west, cyan] at (mb2.north west) {\(\text{MB}_2\)};
    \end{tikzpicture}
    \caption{\textbf{An example of a sample-level causal graph} \textcolor{blue}{event} sequence leading to \textcolor{orange}{outcomes} where \textcolor{blue}{\(\text{MB}_1\)} represents the Markov Boundary of \textcolor{orange}{\(Y_1\)} and \textcolor{cyan}{\(\text{MB}_2\)} the Markov Boundary of \textcolor{red}{\(Y_2\)}. \textcolor{violet}{$X_1$} is the root-cause of \(Y_2\)}
\label{fig:sample_level_graph}
\end{figure}

\begin{figure}[!h]
    \centering
    \includegraphics[width=0.9\linewidth]{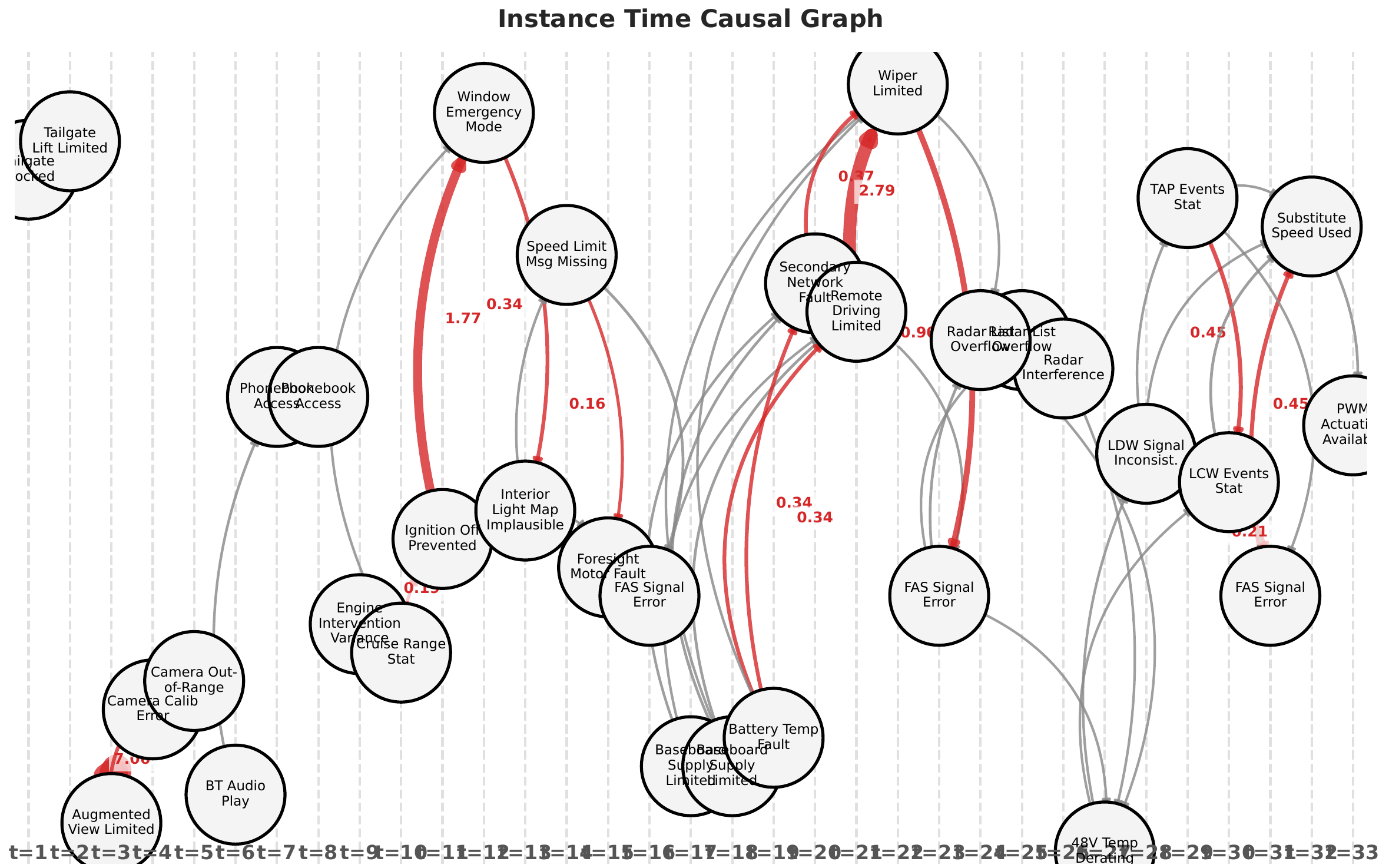}
    \caption{\textbf{Sample Time Causal Graph \(\mathcal{G}^s_{\mathcal{XX}}\).} Temporal evolution of a diagnostic defect cascade in a vehicle ($|\mathcal{X}| \approx 29,100$). \textsc{seq2cause} effectively captures causal relationships, revealing distinct \textbf{error clusters} at different time steps (e.g., initial sensor failures at $t=3$ triggering mechanical faults at $t=12$, battery at issue \(t=17\)). This enables actionable root-cause analysis by isolating the specific onset of a failure mechanism and their strength using The CMI \(\hat{I}_N\)}
    \label{fig:time_instance_graph_dtc}
\end{figure}

\section{Algorithm \& Parallelization}\label{appendix:algorithm_parallelization}
\subsection{Lagged Effects via Simulated Interventions}
To evaluate the lagged effects of an event \(E_{t-\ell}\) on \(E_{t}\), we control for the intermediate events, so-called \textit{mediators} \(\boldsymbol{M} = X_{t-\ell +1:t-1}\) by simulating a Controlled Direct Effect (CDE) \cite{pearl_2009} of \(E_{t-\ell}\) on \(E_{t}\).

\begin{definition}[Randomized Interventional Do-Operator]\label{def:do_in_sequences_expected}
Let \(\boldsymbol{M} = X_{t-\ell+1:t-1} \subset \mathcal{X}\) be the set of intermediate events between cause \(E_{t-\ell}\) and effect \(E_{t}\). We define the intervention \(do(\boldsymbol{M}\sim Q)\) as the expectation over counterfactual realizations sampled from a proposal \(Q\) (e.g., Uniform over \(|\mathcal{X}|\)) and average this effect for \(M\) counterfactual such as:
\begin{align}
  P(E_{t}|do(\boldsymbol{M} \sim Q), x_{<t-\ell}) &\triangleq \mathbb{E}_{\boldsymbol{M} \sim Q} P(E_{t}|\boldsymbol{M}, x_{<{t-\ell}}) \\
    &\approx \frac{1}{M} \sum_{l=1}^{N} P(E_{t} \mid \mathbf{m}^{(l)},  x_{<{t-\ell}})\notag
\end{align}
\end{definition}
\begin{remark}
    This effectively marginalizes out the intermediate causal mechanisms only if we assume that there are no hidden confounders (Assumption~\ref{ass:sufficiency}).
\end{remark}
Using the previous definition, we modify the standard information gain to detect lagged information gain from \(E_{t-\ell}\) to \(E_t\), namely \(I^\mathcal{L}_G\):

\begin{definition}[Lagged Information Gain]\label{def:lagged_ig}
    Let \(E_{t-\ell}\) be the cause, \(E_{t}\) be the effect and \(\boldsymbol{M} = X_{t-\ell+1:t-1}\) the set of intermediate events. The Lagged Information Gain \(I^{\mathcal{L}}_G\) is defined:
\begin{equation}\label{eq:lagged_info_gain}
I^{\mathcal{L}}_{G}(E_{t}; e_{t-\ell} \mid x_{<t-\ell}) \triangleq D_{\mathrm{KL}}\Big(  P\big(E_{t} \mid do(\boldsymbol{M} \sim Q), e_{t-\ell}, x_{< t-\ell}\big) \
\big\| P(E_{t} \mid  do(\boldsymbol{M} \sim Q), x_{< t-\ell}) \Big)
\end{equation}
\end{definition}

\subsection{Parallelization of \textsc{seq2cause} via Teacher Forcing}
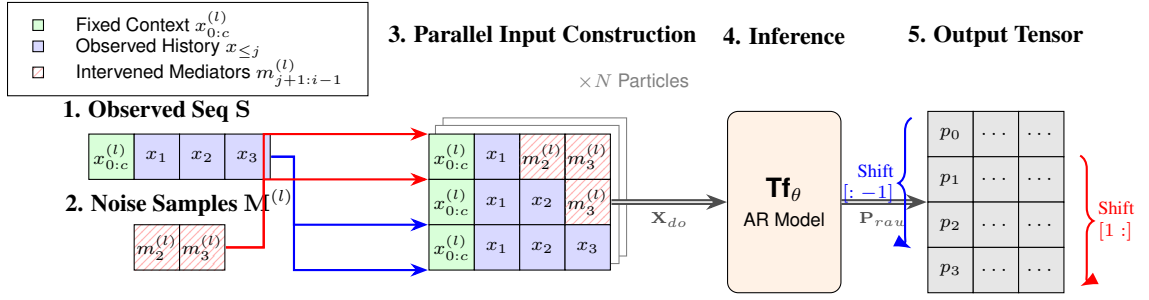
\begin{figure*}[!h]
    \centering
\begin{tikzpicture}[
    font=\sffamily,
    >=Stealth,
    tensor/.style={matrix of nodes, nodes={draw, minimum size=6mm, anchor=center, font=\scriptsize}, column sep=-\pgflinewidth, row sep=-\pgflinewidth, inner sep=0pt},
    obs/.style={fill=blue!15},
    inter/.style={fill=red!15, pattern=north east lines, pattern color=red!30},
    context/.style={fill=green!15},
    gray_out/.style={fill=gray!20},
    label_text/.style={font=\bfseries\small, align=center}
]

\node[label_text] (label_seq) at (-0.3, 4.4) {1. Observed Seq $\mathbf{S}$};
\matrix[tensor] (seq) at (0, 3.8) {
    |[context]| $x^{(l)}_{0:c}$ & |[obs]| $x_1$ & |[obs]| $x_2$ & |[obs]| $x_3$ \\
};

\node[label_text] (label_noise) at (0, 3.2) {2. Noise Samples $\mathbf{M}^{(l)}$};
\matrix[tensor] (noise) at (0, 2.6) {
 |[inter]| $m^{(l)}_2$ & |[inter]| $m^{(l)}_3$  \\
};

\node[label_text] (label_broadcast) at (4.8, 5.4) {3. Parallel Input Construction};


\matrix[tensor] (staircase) at (4.5, 3.2) {
    |[context]| $x^{(l)}_{0:c}$ & |[obs]| $x_1$ & |[inter]| $m^{(l)}_2$ & |[inter]| $m^{(l)}_3$ \\
    |[context]| $x^{(l)}_{0:c}$ & |[obs]| $x_1$ & |[obs]| $x_2$ & |[inter]| $m^{(l)}_3$ \\
    |[context]| $x^{(l)}_{0:c}$ & |[obs]| $x_1$ & |[obs]| $x_2$ & |[obs]| $x_3$ \\
};

\begin{scope}[on background layer]
    \draw[fill=white, draw=black!50] ($(staircase.north west)+(0.2,0.2)$) rectangle ($(staircase.south east)+(0.2,0.2)$);
    \draw[fill=white, draw=black!50] ($(staircase.north west)+(0.1,0.1)$) rectangle ($(staircase.south east)+(0.1,0.1)$);
    \node at (6, 4.8) [text=black!50] {\scriptsize $\times N$ Particles};
\end{scope}

\draw[->, thick, blue] (seq.east) -- ++(0.3,0) |- ($(staircase.west)+(0, -0.3)$); 
\draw[->, thick, blue] (seq.east) -- ++(0.3,0) |- ($(staircase.west)+(0, -0.9)$); 

\draw[->, thick, red] (noise.east) -- ++(0.5,0) |- ($(staircase.west)+(0, 0.9)$); 
\draw[->, thick, red] (noise.east) -- ++(0.5,0) |- ($(staircase.west)+(0, 0.3)$); 

\node[label_text] (label_inference) at (8, 5.4) {4. Inference};
\node[draw, fill=orange!10, minimum width=1.5cm, minimum height=2.4cm, rounded corners, align=center] (model) at (8, 3.2) {\textbf{Tf}$_\theta$ \\ \scriptsize AR Model};

\draw[->, thick, double, black!70] (staircase.east) -- (model.west)
    node[midway, below, font=\tiny\bfseries] {$\mathbf{X}_{do}$};

\node[label_text] at (10.8, 5.4) {5. Output Tensor};

\matrix[tensor] (out_tensor) at (10.8, 3.2) {
    |[gray_out]| $p_0$ & |[gray_out]| $\dots$ & |[gray_out]| $\dots$ \\
    |[gray_out]| $p_1$ & |[gray_out]| $\dots$ & |[gray_out]| $\dots$ \\
    |[gray_out]| $p_2$ & |[gray_out]| $\dots$ & |[gray_out]| $\dots$ \\
    |[gray_out]| $p_3$ & |[gray_out]| $\dots$ & |[gray_out]| $\dots$ \\
};


\draw[->, thick, double, black!70] (model.east) -- (out_tensor.west)
    node[midway, below, font=\tiny\bfseries] {$\mathbf{P}_{raw}$};

\draw[decoration={brace, mirror, amplitude=5pt}, decorate, thick, blue] 
    ($(out_tensor.north west)+(-0.2,-0.1)$) -- ($(out_tensor.south west)+(-0.2, 0.6)$) 
    node[midway, left=3pt, font=\scriptsize, align=right] {Shift \\ $[:-1]$};

\draw[decoration={brace, amplitude=5pt}, decorate, thick, red] 
    ($(out_tensor.north east)+(0.2,-0.6)$) -- ($(out_tensor.south east)+(0.2, 0.1)$) 
    node[midway, right=3pt, font=\scriptsize, align=left] {Shift \\ $[1:]$};

\node[draw, anchor=south east, fill=white] at (2.5, 4.6) {
    \scriptsize
    \begin{tabular}{cl}
         \tikz\node[draw, fill=green!15, inner sep=2pt] {}; & Fixed Context $x^{(l)}_{0:c}$ \\
         \tikz\node[draw, fill=blue!15, inner sep=2pt] {}; & Observed History $x_{\le j}$ \\
         \tikz\node[draw, fill=red!15, pattern=north east lines, pattern color=red!30, inner sep=2pt] {}; & Intervened Mediators $m^{(l)}_{j+1:i-1}$ \\
    \end{tabular}
};

\end{tikzpicture}
\caption{\textbf{Overview of \textsc{seq2cause} Parallel CI-tests.} We construct a single broadcasted tensor $\mathbf{X}_{do}$ where each row $j$ incrementally fixes the history $x_{\le j}$ while randomizing the future (staircase pattern). The model processes this tensor in parallel to produce raw probabilities $\mathbf{P}_{raw}$ (grey). We then compute the CMI by comparing adjacent rows: the distribution at row $j-1$ serves as the baseline ($\mathbf{P}_{base}$, blue) for the intervention at row $j$ ($\mathbf{P}_{do}$, red).}
\label{fig:trace_diagram}
\end{figure*}

\subsection{Scalability of \textsc{seq2cause}}
\paragraph{Sample-Level}
To parallelize the Monte-Carlo estimation (Eq.~\eqref{eq:cmi_estimator}) and counterfactual sampling (Eq.~\ref{eq:lagged_info_gain}) we avoid computing the CMI for the full trajectory \(x_{0:t-1}\) but a truncated version, called \emph{context} as 
\[x_{<t} \approx  x_{0:c} \; \text{for} \; c < t \; \text{and} \; 0 < c \ll L\]
We argue that it makes it possible to parallelize the CI-tests by having one common dimension for the sampling and renders the problem feasible on GPUs. We usually take \(c = \text{max}(0.1L, 20)\) in our experiments. Although it might break Markovianity for long sequence, empirical results show robustness to this truncation. We provide ablation to unseen sequence lengths during training and show our method to be robust to high delayed effects in Fig~\ref{fig:main_results}. 
\paragraph{Sparse Approximation for lagged effect in \(\mathcal{X}\to \mathcal{X}\).}\label{app:sparse_aprox}
For each time step \(t\), we must perform \(t\) CI-tests (one for every potential lag per step). Thus, for a sequence of length \(L\), the total number of CI-tests is given by \(\sum^L_{t=1} t = \frac{L(L+1)}{2}\) which grows quadratically with the sequence length. As a result, even on multiple GPUs, it becomes tricky to infer for long sequences \(L > 100\).

To solve this, we propose a \emph{sparse variant} for which we bound the lagged effects of previous events on future events up to a memory \(m\). Thus the DGP \(\{X_t\}\) becomes an \(m\)-order Markov chain. 
With \(m \ll L\), \textsc{seq2cause} scales linearly with the sequence length \(L\). The memory complexity transitions from:
\begin{equation*}\label{eq:linearization_complexity}
    \mathcal{O}( N \cdot (L-c) \cdot L \cdot |\mathcal{X}|) \xrightarrow{\text{Bounded Memory}} \mathcal{O}(N \cdot m \cdot L \cdot |\mathcal{X}|)
\end{equation*}

\paragraph{Population-Level.}
Population-level inference via the MDL criterion (Eq.~\ref{eq:mdl_delta}) 
is strictly cheaper than sample-level inference: it requires \emph{no 
Monte-Carlo estimation}. For each matched position 
$(k,t) \in \mathcal{M}_{uv}$, the per-step log-ratio $\ell_t^{(k),\theta}(u,v)$ 
reduces to two forward passes through the frozen AR model — one at 
$Z_u^{(k,t)}$ (cause present) and one at $Z_\varnothing^{(k,t)}$ (cause 
absent) — whose outputs are directly compared as a log-ratio. No particle 
sampling, no history simulation, no stochastic approximation: the total 
accumulated evidence $\hat{\Delta}_m(u,v)$ is obtained by summing 
$n_{uv}$ (number of event pairs $u \to v)$ such log-ratios deterministically across the dataset, at a 
memory complexity of:
\begin{equation*}\label{eq:memory_complexity}
    \mathcal{O}(L^2 \cdot M \cdot |\mathcal{X}|)
\end{equation*}
where $M$ is the number of sampled mediators required to evaluate 
the do-operator (empirically $M{=}4$ suffices, 
Appendix~\ref{fig:ablation_m_mediators_graph_pop}). For the 
self-corrected threshold $\tau_{uv}^*$ (Cor.~\ref{cor:adaptive_threshold}), 
the error correction term $\sqrt{\epsilon_T/2} \cdot \eta_{uv}$ requires 
the local binary approximation error $\epsilon_T$ at each 
matched position. This quantity is obtained at \emph{zero additional cost}: 
recalling that the cross-entropy loss decomposes as 
$\mathcal{L}_{\mathrm{AR}}(\theta) = H(P) + D_{\mathrm{KL}}(P \| P_\theta)$, 
the generative error $\epsilon_T$ is read off directly as 
the binary cross-entropy minus the entropy floor 
$H(P)$. No additional model 
queries are required. This inference is fully parallelizable on GPUs because it uses the parallel lagged-effect implementation shown in Fig.~\ref{fig:trace_diagram}.


\section{Evaluation}\label{appendix:evaluation}
\paragraph{Synthetic experiments settings.}
We used an \(ml.g5.4xlarge\) instance from AWS Sagemaker, which contains 8 vCPUs and 1 NVIDIA A10G as GPU with 24GiB for training and inference.
\paragraph{Vehicle diagnostics experiments settings.}
We used an \(ml.g5.12xlarge\) instance from AWS Sagemaker, which contains 8 vCPUs and 4 NVIDIA A10G as GPU with 24GiB for inference.

\subsection{Nonlinear SCMs}
As we saw, standard causal discovery algorithms for multivariate time series (e.g., PCMCI, Neural Hawkes Processes, CASCADE) are not applicable in our setting. Moreover, evaluating causal discovery in high-dimensional event sequences is notoriously difficult due to the lack of ground-truth annotations in real-world traces (e.g., server logs, medical records). Furthermore, generic high-order MC are computationally intractable to materialize due to the state space \(|\mathcal{X}|^m\) of \(\{X_t\}\) and unlearnable in high dimensions without structural assumptions.

\paragraph{Synthetic Data.}\label{appendix:scm}
We introduce a synthetic benchmark based on a nonlinear Structural Causal Model (SCM). 
Let $\mathcal{X}$ be the set of event types,
$s = (x_1,\dots,x_L) \in \mathcal{X}^{L}$ a discrete sequence, and
$h$ the history window.
The SCM generative process is
\begin{equation}
  P(X_t \mid X_{t-h:t-1})
  = \operatorname{softmax}\!\left(
      \mathbf{b}
      + \sum_{k=1}^{h} e^{-(k-1)}\,\mathbf{W}[x_{t-k}]
      + \operatorname{ReLU}\!\left(
          \bigl[E_{x_{t-h}},\dots,E_{x_{t-1}}\bigr]\,\mathbf{W}_1
        \right)\mathbf{W}_2
    \right),
  \label{eq:scm}
\end{equation}
where $\mathbf{W}\in\mathbb{R}^{|\mathcal{X}|\times |\mathcal{X}|}$ is a sparse interaction matrix
(sparsity $= 0.9$, mixed-sign), $E\in\mathbb{R}^{|\mathcal{X}|\times 16}$ are fixed random
embeddings, and $\mathbf{W}_1\in\mathbb{R}^{h\cdot 16\times 64}$,
$\mathbf{W}_2\in\mathbb{R}^{64\times |\mathcal{X}|}$ form a one-hidden-layer MLP
that captures non-linear interactions.
The ground-truth adjacency is determined by counterfactual
intervention~\citep{pearl_2009}: $i\to j$ if
$\mathbb{E}\bigl[\mathrm{KL}(P_{\mathrm{orig}}\!\parallel P_{\mathrm{do}(x_{t-k}=\epsilon)})\bigr]>\delta$.
We tune the sparsity of the weight matrix \(\mathbf{W}\) to obtain the Shannon redundancy~\cite{shannon1951prediction} as \( 1-\frac{H(P)}{H_{max}} = 1-\frac{H(P)}{\log(|\mathcal{X}|)}\) superior or equal to \(58\%\) across the benchmarked SCMs. 

\paragraph{Evaluation via intervention.}
We always evaluate \textsc{seq2cause} on the summary graph. To measure performance, we perform atomic interventions by measuring the average KL divergence over 10 by uniformly randomizing \(E_{t-\ell}\) and measuring the average KL divergence over 10 counterfactual between post-intervention and observational distributions of \(E_{t}\). If the divergence is above \(\tau> 0.05\), an edge \(E_t \rightarrow E_{t-\ell}\) exists in the summary graph.

\subsection{\(\epsilon\) in Practice}\label{app:epsilon_monitor}

\paragraph{What \(\epsilon\) stands for.} 
\(\epsilon\) replaces parametric structural assumptions (linearity, Gaussianity, 
additive noise) with a single monitorable quantity: the excess cross-entropy
\begin{equation}\label{eq:epsilon_with_ar_loss}
\epsilon \;=\; \mathcal{L}_{\mathrm{AR}}(\theta) - H(P),
\end{equation}
measurable during AR pretraining without ground-truth causal labels.

\paragraph{Practical estimation.} 
Two quantities in \eqref{eq:epsilon_with_ar_loss} need handling:

\emph{(i) The cross-entropy \(\mathcal{L}_{\mathrm{AR}}(\theta)\)} is the standard 
held-out validation loss, directly available from training diagnostics.

\emph{(ii) The entropy floor \(H(P)\)} is inaccessible in practice. We approximate it 
by the minimum held-out validation loss achievable as training converges, under the 
standard assumption that \(\mathcal{L}_{\mathrm{AR}}(\theta^\ast) \to H(P)\) as 
capacity and data grow. Empirically we find \(\epsilon < 0.05\) is sufficient for 
structural recovery (Fig.~\ref{fig:robustness_and_scale}(a), phase transition at 
\(\epsilon \approx 0.1\)).


\paragraph{Comparing \(\epsilon\) across vocabularies: the Oracle Score.} 
Different DGPs have different entropy floors \(H(P)\), making raw \(\epsilon\) 
incomparable across settings. For the nonlinear-SCM experiments 
(Fig.~\ref{fig:robustness_and_scale}(b)) we use the \emph{Oracle Score} \(\hat\epsilon\) 
as a normalized estimator of the excess entropy:
\begin{equation}\label{eq:oracle_score}
\hat\epsilon(P_\theta) \;=\; \frac{\mathcal{L}_{\mathrm{AR}}(\theta) - H(P)}
                                   {H_{\max} - H(P)},
\end{equation}
where \(H_{\max} = \log|\mathcal{X}|\) is the maximum entropy (uniform noise). 
\(\hat\epsilon \to 0\) implies convergence to the theoretical limit 
(\(P_\theta \to P\)); \(\hat\epsilon = 1\) is the uniform prior.


\subsection{Population-Level \(\mathcal{X} \to \mathcal{X} \) baselines}
\label{sec:baselines}

We compare \textsc{Seq2Cause} against several families of baselines for the \(\mathcal{X} \to \mathcal{X}\) regime that span
association-based, information-theoretic, cascade-based, point-process, and
neural attribution methods. Like for \textsc{seq2cause} we apply an individual significance threshold \(\tau\) to the estimated adjacency matrix to binarize it. Importantly, \emph{even when we grant PMI and TF-IDF oracle access} to the threshold that maximises F1 on the test set — an advantage unavailable in practice — \textsc{seq2cause} still outperforms them (Appendix.~\ref{fig:baseline_pop_x_x_oracle_thres}). In contrast, \textsc{seq2cause} requires no threshold calibration: its MDL criterion derives a principled decision boundary directly from the density error bound.


\paragraph{Pointwise Mutual Information (PMI).}
For each ordered pair $(i,j)$, we count co-occurrences within a lag window of
size $h$:
\begin{equation}
  \widehat{\mathrm{PMI}}(i\to j)
  = \log\frac{\hat{P}(x_t=j,\, x_{t-k}=i \text{ for some } k\le h)}
             {\hat{P}(x_{t-k}=i)\,\hat{P}(x_t=j)},
\end{equation}
clipped at zero (positive PMI).
PMI measures symmetric association; it is included as a
simple count-based reference and does not exploit temporal order beyond the
lag constraint~\citep{church-hanks-1990-word}.

\paragraph{TF-IDF.}
We adapt the classical retrieval score~\citep{tfidf} for identifying causal pairs over a discrete event vocabulary $\mathcal{V}$.
Treating each source type $i$ as a \emph{document} whose terms are the event
types that follow it within $h$ steps, we define
\begin{equation}
  \mathrm{Score}(i \to j)
  \;=\;
  \underbrace{\frac{\mathrm{count}(i \leadsto j)}{\mathrm{count}(i)}}_{\mathrm{f}(i\to j)}
  \;\times\;
  \underbrace{\log\!\left(\frac{|\mathcal{V}|+1}
    {|\{i' : \mathrm{count}(i' \leadsto j) > 0\}|+1}\right)}_{\mathrm{IDF}(j)},
\end{equation}
where $\mathrm{count}(i \leadsto j)$ is the number of times type $j$ appears
within $h$ steps after type $i$ across all observed sequences.
The TF term measures how reliably $i$ is followed by $j$; the IDF term
down-weights effect types $j$ that are triggered by many distinct sources,
which in this setting correspond to high-frequency background events with no
specific cause.
Compared to PMI, which penalises both rare causes and rare effects, IDF
operates in log-count space and only suppresses effects that lack causal
specificity, making it more robust when many event types are infrequent.
An edge $i\to j$ is predicted whenever $\mathrm{Score}(i\to j)>\tau$.

\begin{figure}
    \centering
    \includegraphics[width=1\linewidth]{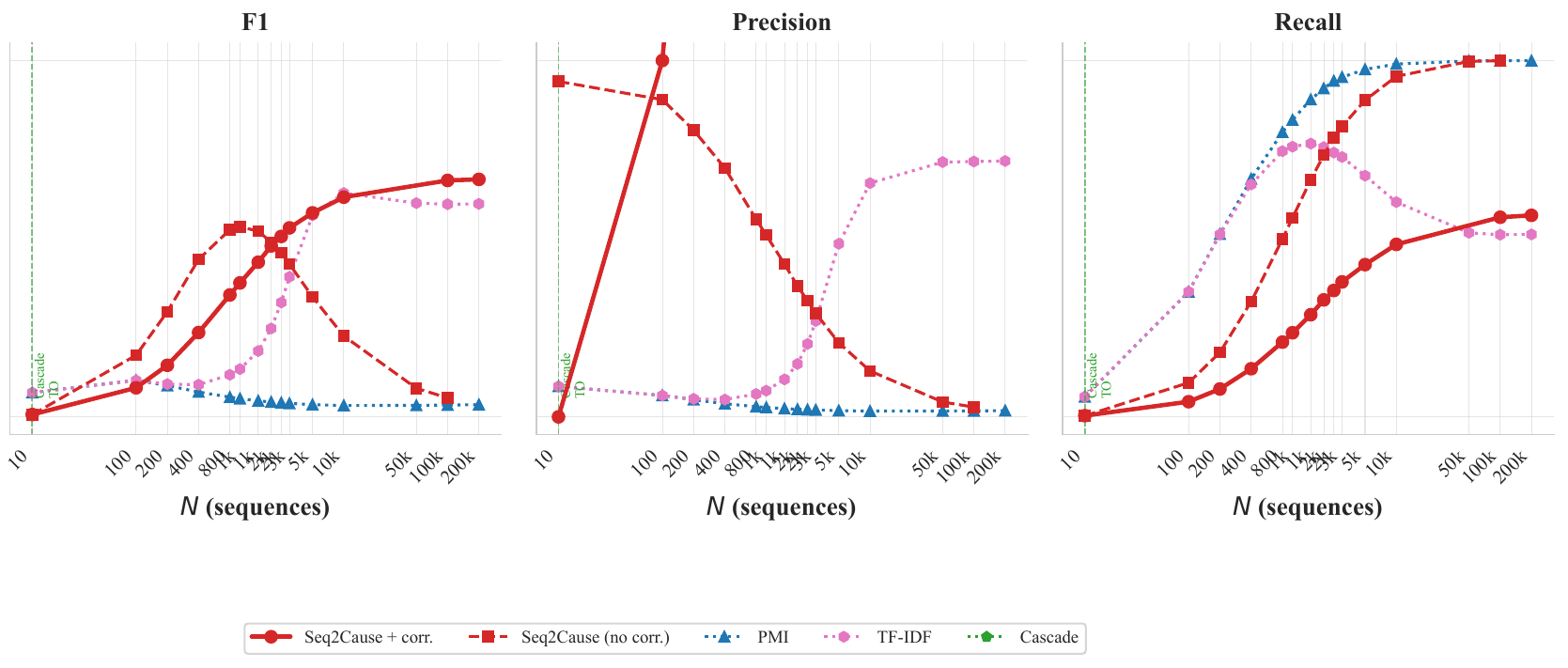}
\caption{
  \textbf{Population-level \(\mathcal{X} \to \mathcal{X}\) baseline comparison
  with oracle threshold selection.}
  F1, Precision, and Recall as a function of the number of observed sequences
  $N$ for each method on the nonlinear SCM benchmark
  ($|\mathcal{V}|=1000$, $m=6$).
  For PMI and TF-IDF, scores are binarised at the threshold $\tau^*$ that
  \emph{maximises F1 on the held-out ground truth}; this constitutes an upper
  bound on their achievable performance, since such oracle calibration is
  unavailable at test time.
  \textsc{seq2cause} uses its MDL-derived threshold
  (Corollary~\ref{cor:adaptive_threshold}) with no access to the ground truth.
  Despite this advantage, PMI and TF-IDF remain below \textsc{seq2cause}
  across all sample sizes, confirming that the performance gap is not an
  artefact of threshold miscalibration but reflects a fundamental limitation
  of association-based scores for causal discovery.
}    \label{fig:baseline_pop_x_x_oracle_thres}
\end{figure}

\paragraph{Cascade (MDL-based).}
The Cascade method~\citep{cueppers2024causal} frames causal discovery as a
minimum description length (MDL) problem: it selects the directed graph $G$
that minimises the two-part code $L(G) + L(\mathbf{x}\mid G)$, where the
second term encodes the observed event sequences given $G$.
An event $i$ is declared a cause of $j$ if adding the directed edge $i\to j$
compresses the data by more than the cost of encoding the edge.
We use the reference implementation of \citet{cueppers2024causal}\footnote{%
  \url{https://github.com/joschac/cascade}} with \texttt{max\_delay}$=h$ and
\texttt{allow\_instant}$=\mathrm{False}$, treating each sequence as an
independent stream separated by a temporal gap $>h$ to prevent spurious
cross-sequence edges.
Cascade returns a binary adjacency matrix directly; no threshold is applied.

\paragraph{SHTP; Structural Hawke Processes.}
We use the implementation of \citet{shtp} for SHTP\footnote{ \url{https://github.com/DMIRLAB-Group/SHP}} that model the event apparition as a Hawkes Process. We treat each sequence as an independent stream exactly like Cascade.

\subsubsection{Differentiation with Granger}\label{appendix:granger_vs_seq2cause}
\textsc{seq2cause} works like a constraint based algorithm~\cite{constrainct_based_cd} whereas Granger causality's core logic remains on predictability.

First, Granger causality assumes a fixed parametric form 
(linear VAR) for the DGP and requires stationarity; 
\textsc{seq2cause} assumes neither, replacing structural assumptions 
with the single monitorable quantity $\epsilon$ (Assumption~\ref{ass:oracle}).
Second, for lagged effects ($\ell > 1$) in the 
$\mathcal{X}\to\mathcal{X}$ regime, \textsc{seq2cause} performs a 
do-intervention on intermediate mediators rather than 
conditioning on them, which distinguishes direct from 
total effects — a distinction Granger methods cannot make 
without additional structural assumptions~\cite{pearl_1998_bn}.
Third, under causal sufficiency (A\ref{ass:sufficiency}), the full-history conditioning set $X_{<t}$ satisfies the back-door criterion 
for every candidate cause $E_{t-\ell}$, so the CMI test 
identifies the direct causal effect rather than a 
regression residual.
The empirical superiority over Neural Granger (same backbone, 
probability difference instead of CMI, $\Delta$F1 $>$20 
points, Tab.~\ref{tab:unified}) isolates the contribution of the 
interventional CMI signal over regressing on predictive 
residuals — confirming that the difference is not merely 
architectural.

\subsection{Ablations}
\begin{figure}[!h]
\centering
\includegraphics[width=\linewidth]{%
    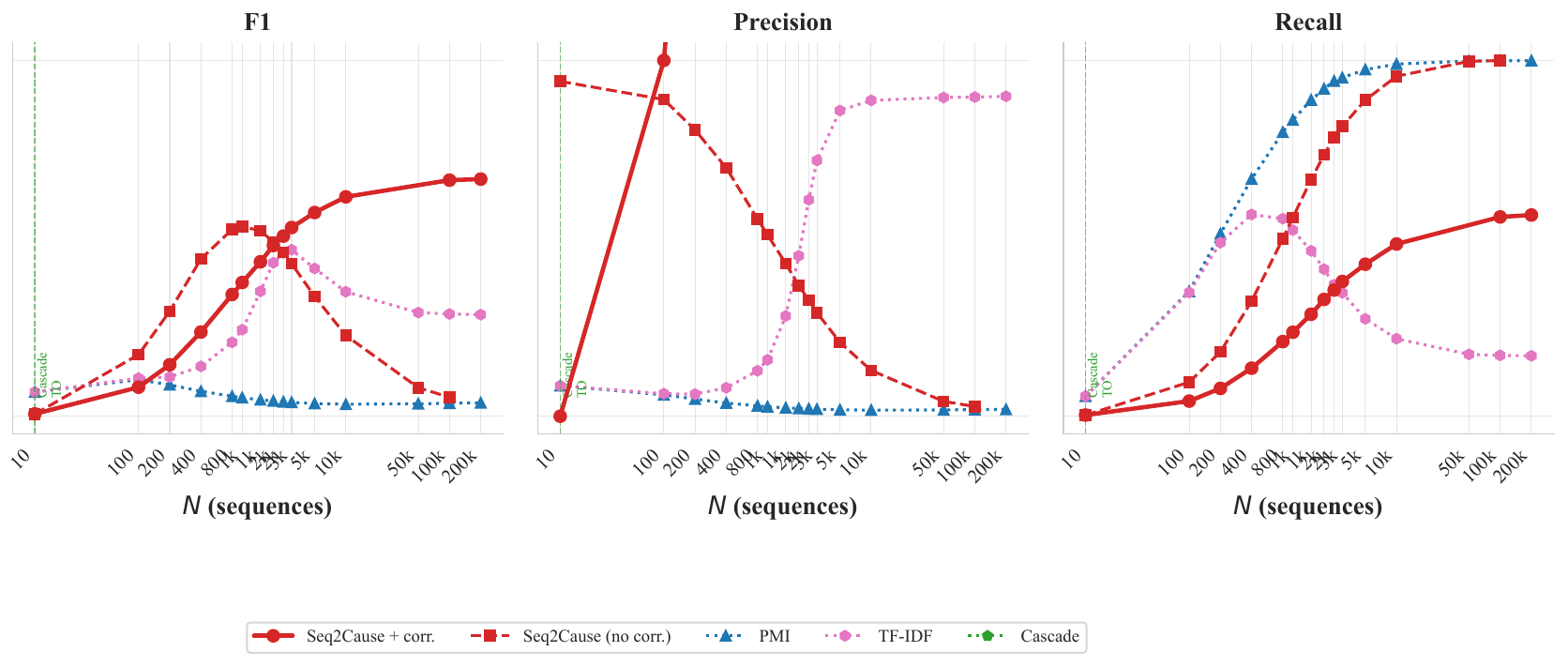}
\caption{%
  \textbf{Population-level $\mathcal{X}\!\to\!\mathcal{X}$ 
  discovery as a function of number of sequences $n$}
  ($|\mathcal{X}|{=}1{,}000$, nonlinear-SCM, 
  $\epsilon{=}0.05$).
  The green dashed line marks the CASCADE timeout ($N{=}100$);
  SHP times out at the same scale and is omitted beyond that point.
  \textit{Correlation baselines} (PMI, TF-IDF): Recall saturates 
  near 1.0 early as heuristic thresholds admit all co-occurring 
  pairs, while Precision remains near zero — confirming that 
  frequency-based methods cannot distinguish causal structure 
  from spurious association regardless of evidence volume.
  \textit{\textsc{seq2cause} w/o correction} (dashed red): 
  Precision peaks near $N{=}200$ then collapses to 4.2\% as 
  oracle approximation errors $\sqrt{\epsilon_T/2}\cdot\eta_{uv}$ 
  accumulate without correction, empirically validating 
  Lem.~\ref{lem:mdl_deviation} — the uncorrected threshold 
  becomes equivalent to no threshold at scale.
  \textit{\textsc{seq2cause} w/ correction} (solid red): the 
  adaptive $\tau^*_{uv}$ absorbs error accumulation, maintaining 
  Precision above 0.80 across the full range while Recall grows 
  steadily as genuine weak causal signals clear the corrected 
  threshold — the only method achieving balanced F1 at any scale.
  ($L{=}64$, $m{=}6$, $\tau^*_{uv}$ from 
  Cor.~\ref{cor:adaptive_threshold}.)
}
\label{fig:nb_seq_baseline}
\end{figure}

\subsubsection{Robustness analysis}
\begin{figure*}[!h]
    \centering
    \includegraphics[width=1\linewidth]{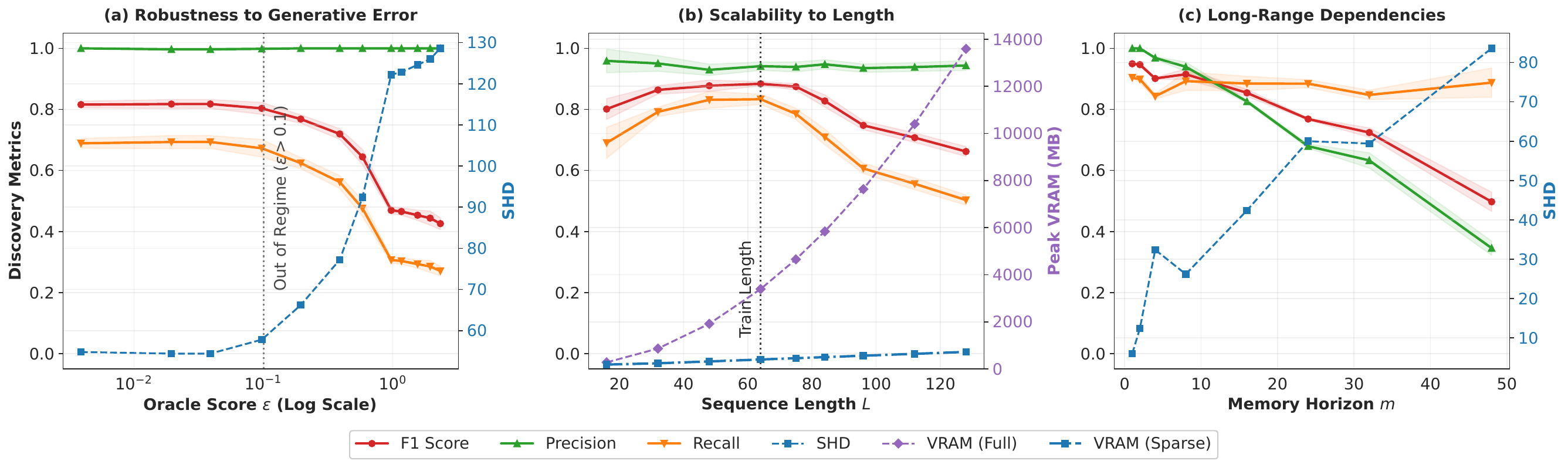}
    \caption{\textbf{Sample-Level (\(\mathcal{X} \to \mathcal{X}\)) Robustness and Scalability Analysis} ($|\mathcal{X}|=1000, N=128, \tau=10^{-4}, L=64$). \textbf{(a) Robustness to Generative Error:} Performance as a function of the model's oracle score $\epsilon$. \textsc{seq2cause} exhibits a phase transition, recovering structure \emph{even for imperfect models} ($\epsilon < 0.1$) and maintaining high Precision even as fidelity degrades. \textbf{(b) Scalability to Length:} Performance and GPU memory usage vs. sequence length $L$. The \textbf{Sparse} variant (Appendix. \ref{app:sparse_aprox}) demonstrates linear memory scaling ($O(mL)$), enabling inference on sequences far exceeding the training length ($L=64$), whereas the Full variant scales quadratically. \textbf{(c) Long-Range Dependencies:} Robustness to increasing delayed-effects $m$. \textsc{seq2cause} maintains F1 $>0.8$ even as dependencies span one third of the sequence ($m=20$), confirming the method's ability to capture distant causal mechanisms.}
    \label{fig:main_results}
\end{figure*}

\paragraph{Scalability to sequence length.}
Figure~\ref{fig:main_results}(b) evaluates performance
and GPU memory on sequences beyond the AR model's
training window.
The sparse variant ($m$-bounded lags) achieves
$\mathcal{O}(m \cdot L \cdot |\mathcal{X}|)$ memory
versus $\mathcal{O}(L^2 \cdot |\mathcal{X}|)$ for
the full variant, enabling inference on sequences
far exceeding the training length.
Precision remains stable (${\approx}0.95$) while
Recall degrades gracefully, again confirming the
conservative failure mode.

\paragraph{Deep temporal dependencies.}
Figure~\ref{fig:main_results}(c) stress-tests the
parallelised intervention mechanism by extending
the SCM memory horizon to $m\!=\!48$.
\textsc{seq2cause} maintains F1\,$>\!0.80$ up to
$m\!=\!20$ (one third of the sequence length),
confirming effective capture of long-range causal
signals without exponential degradation.

\begin{table}[!h]
\centering
\caption{\textbf{Double degeneracy of classical MB algorithms on the Sample-Level $\mathcal{X}\!\to\!\mathcal{Y}$.}
Weighted F1 (\%) on the vehicle diagnostic Markov-boundary recovery task, evaluated on restricted vocabularies ($|\mathcal{X}|$ = top-$K$ most frequent event types) and reduced sample sizes $n$. At low $(K,n)$, classical LSL algorithms run to completion but collapse to $\text{F1}=0$ because expected co-occurrence $\mathbb{E}[\#(u,Y_j)] \approx n/(K|\mathcal{Y}|)$ falls below the threshold required for any asymptotic CI test to reject independence. Scaling up either $K$ or $n$ to alleviate this sample-starvation regime immediately triggers the
dual vocabulary-driven intractability regime ($^{\dagger}$: timeout $>$3 days on 4$\times$A10G) GPUs.
\textsc{seq2cause} amortizes conditional density estimation across all $(u, Y_j)$ pairs through a
single AR forward pass and is the only method viable in both regimes.}
\label{tab:reduced_n_mb}
\small
\begin{tabular}{ll|ccccc|c}
\toprule
$|\mathcal{X}|$ & $n$ & IAMB & CMB & MB-by-MB & PCDbyPCD & MI-MCF & \textsc{seq2cause} \\
\midrule
500    & 500    & 0.0 & 0.0 & 0.0 & 0.0 & 0.0 & \textbf{58.4} \\
500    & 5,000  & 2.1 & 1.8 & 2.4 & 0.0 & 0.0 & \textbf{59.7} \\
2,000  & 500    & 0.0 & 0.0 & 0.0 & 0.0 & 0.0 & \textbf{52.2} \\
2,000  & 5,000  & 0.0 & 0.0 & 0.0 & --$^{\dagger}$ & 0.0 & \textbf{53.1} \\
2,000  & 50,000 & --$^{\dagger}$ & --$^{\dagger}$ & --$^{\dagger}$ & --$^{\dagger}$ & --$^{\dagger}$ & \textbf{55.3} \\
29,100 & 50,000 & --$^{\dagger}$ & --$^{\dagger}$ & --$^{\dagger}$ & --$^{\dagger}$ & --$^{\dagger}$ & \textbf{46.1} \\
\bottomrule
\end{tabular}
\end{table}

\subsubsection{Population-Level: Number of sequence \(n\)}
\begin{figure}[t]
\centering
\includegraphics[width=0.8\linewidth]{%
    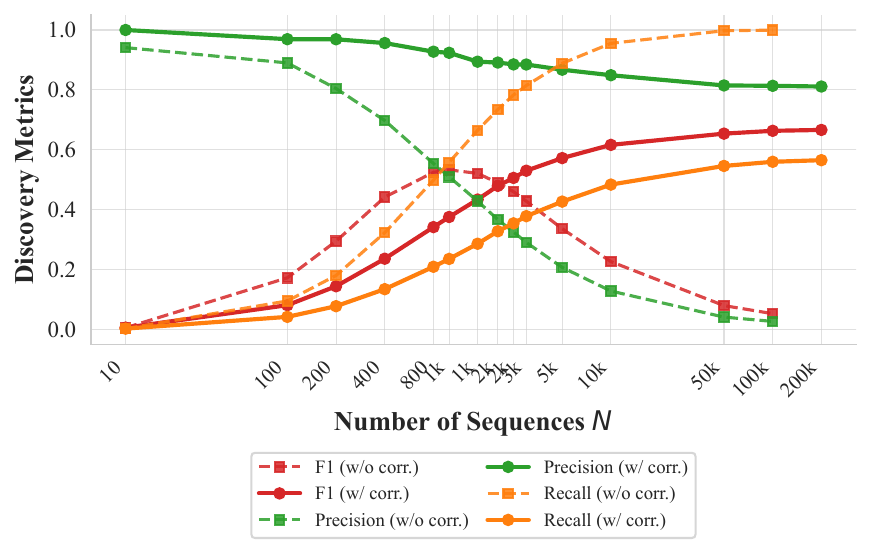}
\caption{%
  \textbf{Self-correcting MDL threshold $\tau^*_{uv}$ prevents
  precision collapse as evidence accumulates.}
  Population-level causal discovery metrics as a function of the
  number of sequences $N$, comparing the uncorrected MDL threshold
  (dashed, \textit{w/o corr.}) against the generative-error-corrected
  threshold of Cor.~\ref{cor:adaptive_threshold}
  (solid, \textit{w/ corr.}).
  \textit{Without correction:} Precision degrades catastrophically
  beyond $N{\approx}1{,}000$ sequences — each additional sequence
  accumulates approximation error from the $\epsilon$-oracle,
  progressively lowering the effective threshold until spurious
  edges flood the graph (Recall $\to 1$, Precision $\to 0$).
  \textit{With correction:} The adaptive term
  $\sqrt{\epsilon_T/2}\cdot\eta_{uv}$ in $\tau^*_{uv}$
  absorbs this drift, maintaining Precision above 0.80 across
  the full range $N \in [10, 200{,}000]$ while Recall grows
  steadily as genuine weak causal signals accumulate evidence.
  The crossover at $N{\approx}2{,}000$ marks the point where
  uncorrected evidence accumulation overtakes the model cost
  $c_{\mathrm{type}}$ (synthetic non-linear SCMs: $|\mathcal{X}|{=}1{}000, L=64$).
}
\label{fig:ablation_correction}
\end{figure}
\subsubsection{Sample-Level: Number of Particles \(N\)}
We observe in Fig.~\ref{fig:ablation_n_particles} that the number of particles \(N\) is a crucial parameter to increase Recall and reduce SHD. After \(N > 256\), however, we don't observe significant changes. 
\begin{figure}[!h]
    \centering
\includegraphics[width=0.65\linewidth]{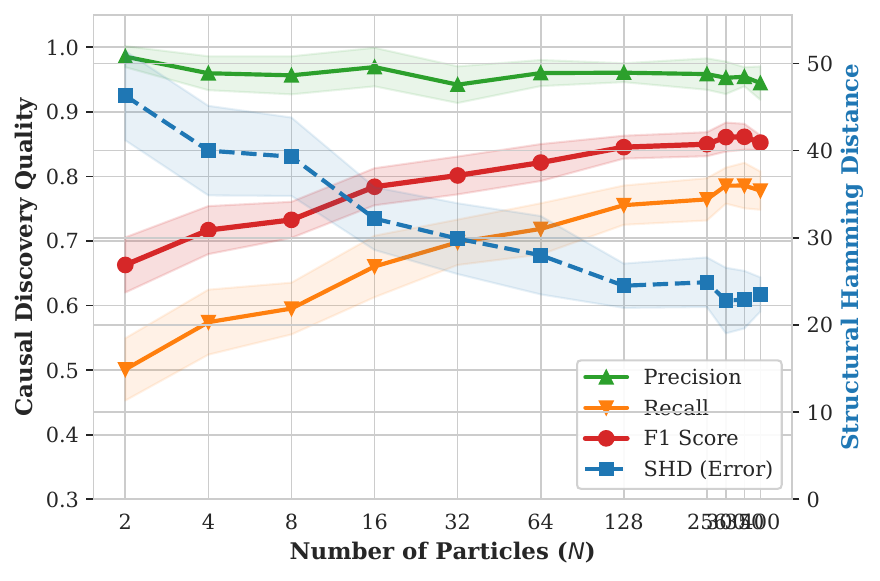}
    \caption{\textbf{Evolution of the Sample-Level \(\mathcal{X} \to \mathcal{X}\) Causal Discovery in Function of the Number of Particles $N$} at \(|\mathcal{X}| = 1000, m=6, L=64\) on non-linear SCMs.}
\label{fig:ablation_n_particles}
\end{figure}

\subsubsection{Population-Level: Number of Counterfactuals \(M\) for Mediators}
We ablate the number of counterfactuals for the mediators \(M\) in Fig.~\ref{fig:ablation_m_mediators_graph_pop}. We observe that we don't need exceeding number of counterfactuals and that generally 2 to 4 is sufficient.

\begin{figure}[!h]
    \centering
\includegraphics[width=0.7\linewidth]{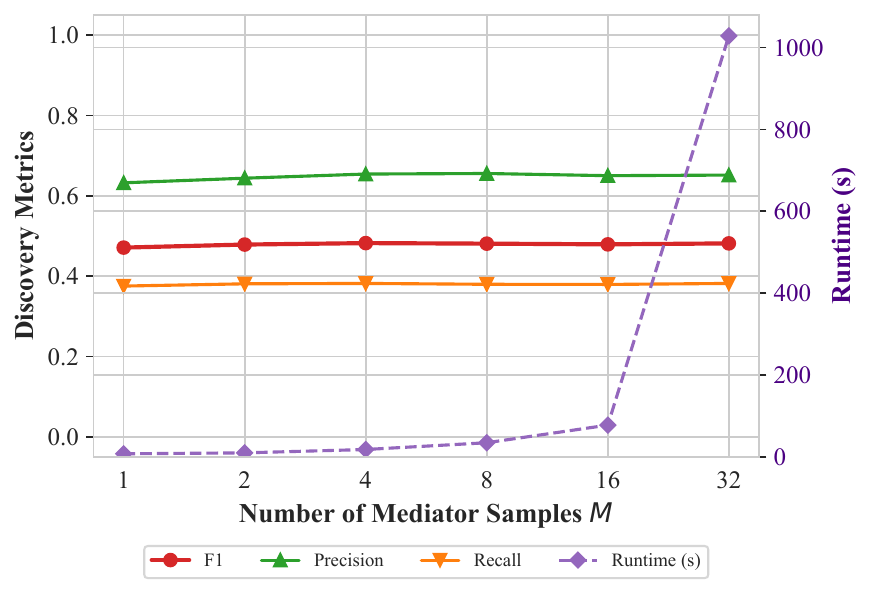}
    \caption{\textbf{Evolution of the Population Causal Discovery in Function of the Number of Counterfactuals sampled as Mediators $M$} at \(|\mathcal{X}| = 1000, m=6, L=66, n = 2000\)}
\label{fig:ablation_m_mediators_graph_pop}
\end{figure}


\newpage
\section*{NeurIPS Paper Checklist}

The checklist is designed to encourage best practices for responsible machine learning research, addressing issues of reproducibility, transparency, research ethics, and societal impact. Do not remove the checklist: {\bf The papers not including the checklist will be desk rejected.} The checklist should follow the references and follow the (optional) supplemental material.  The checklist does NOT count towards the page
limit. 

Please read the checklist guidelines carefully for information on how to answer these questions. For each question in the checklist:
\begin{itemize}
    \item You should answer \answerYes{}, \answerNo{}, or \answerNA{}.
    \item \answerNA{} means either that the question is Not Applicable for that particular paper or the relevant information is Not Available.
    \item Please provide a short (1--2 sentence) justification right after your answer (even for \answerNA). 
\end{itemize}

\end{document}